\documentclass{article}
\PassOptionsToPackage{numbers, compress}{natbib}

\usepackage[final]{neurips_2021}

\usepackage[utf8]{inputenc} %
\usepackage[T1]{fontenc}    %
\usepackage{hyperref}       %
\usepackage{url}            %
\usepackage{booktabs}       %
\usepackage{amsfonts}       %
\usepackage{nicefrac}       %
\usepackage{microtype}      %
\usepackage[table,dvipsnames]{xcolor}         %

\usepackage{graphicx}
\usepackage{amsmath,amssymb} %
\usepackage{amsthm}
\usepackage{color}
\usepackage{array}
\usepackage{subcaption}
\usepackage{xspace}
\usepackage{footnote}
\usepackage{wrapfig}
\usepackage[ruled,vlined,linesnumbered]{algorithm2e} 
\usepackage{wrapfig}
\usepackage{tabularx}
\usepackage{multirow}
\usepackage{colortbl}
\usepackage{cleveref}

\newcommand{\fexp}{f^{\text{teach}}}
\newcommand{\piexp}{\pi^{\text{teach}}}

\newcommand{\expert}{teacher\xspace}

\newcommand{\thor}{\textsc{AI2-THOR}\xspace}
\newcommand{\robothor}{\textsc{RoboTHOR}\xspace}
\newcommand{\habitat}{\textsc{AIHabitat}\xspace}

\newcommand{\feasiblepi}{\Pi^{\text{feas.}}}
\usepackage{commands}
\input{commands.tex}

\newcommand{\rebuttal}[1]{{\color{black}#1}}

\definecolor{ColorPNav}{HTML}{F5F5F5}
\definecolor{ColorFurnMove}{HTML}{FFF8D7}
\definecolor{ColorFootball}{HTML}{EFFFE8}

\newcommand{\pnav}{\mbox{\sc{{PointNav}}}\xspace}
\newcommand{\onav}{\mbox{\sc{{ObjectNav}}}\xspace}
\newcommand{\cnav}{\mbox{\sc{{CoopNav}}}\xspace}

\newtheorem{prop}{Proposition}

\theoremstyle{definition}
\newtheorem{example}{Example}

\definecolor{ColorPNav}{HTML}{F5F5F5}
\definecolor{ColorFurnMove}{HTML}{FFF8D7}
\definecolor{ColorFootball}{HTML}{EFFFE8}

\title{Bridging the Imitation Gap \\by Adaptive Insubordination}

\author{%
  Luca Weihs\thanks{denotes equal contribution by LW and UJ; $^{\dagger}$work done, in part, as an intern at Allen Institute for AI}\hspace{1.5mm}$^{,1}$,\hspace{1mm}
  Unnat Jain$^{*,2,\dagger}$,\hspace{1mm}
  Iou-Jen Liu$^{2}$,\hspace{1mm}
  Jordi Salvador$^{1}$,\\
  \textbf{Svetlana Lazebnik$^{2}$,
  Aniruddha Kembhavi$^{1}$,\hspace{1mm}
  Alexander Schwing$^{2}$}\\
  $^{1}$Allen Institute for AI,\hspace{2mm}$^{2}$University of Illinois at Urbana-Champaign \\
  \small{\texttt{\{lucaw, jordis, anik\}@allenai.org}}\\ \small{\texttt{\{uj2, iliu3, slazebni, aschwing\}@illinois.edu}} \\
  {\color{magenta}\url{https://unnat.github.io/advisor/}}
}

\begin{document}

\maketitle

\begin{abstract}
In practice, imitation learning is preferred over pure reinforcement learning whenever it is possible to design a teaching agent to provide expert supervision. However, we show that when the teaching agent makes decisions with access to privileged information that is unavailable to the student, this information is marginalized during imitation learning, resulting in an ``imitation gap'' and, potentially, poor results. Prior work bridges this gap via a progression from imitation learning to reinforcement learning. While often successful, gradual progression fails for tasks that require frequent switches between exploration and memorization. To better address these tasks and alleviate the imitation gap we propose `Adaptive Insubordination' (ADVISOR). ADVISOR dynamically weights imitation and reward-based reinforcement learning losses during training, enabling on-the-fly switching between imitation and exploration. On a suite of challenging tasks set within gridworlds, multi-agent particle environments, and high-fidelity 3D simulators, we show that on-the-fly switching with ADVISOR outperforms pure imitation, pure reinforcement learning, as well as their sequential and parallel combinations.
\end{abstract}

\section{Introduction}\label{sec:introduction}
Imitation learning (IL) can be remarkably successful in settings where reinforcement learning (RL) struggles. For instance, IL has been shown to succeed in complex tasks with sparse rewards \citep{chevalier2018babyai,DeepMimicPengEtAl2018,NairMcGrewEtAl2018}, and when the observations are high-dimensional, \eg, in visual 3D environments \citep{ai2thor,habitat19iccv}. %
To succeed, IL provides the agent with consistent expert supervision at every timestep, making it less reliant on the agent randomly attaining success.
To obtain this expert supervision, it is often convenient to use ``privileged information,'' \ie, information that is unavailable to the student at inference time. This privileged information takes many forms in practice. For instance, in navigational tasks, experts are frequently designed using shortest path algorithms which access  the environment's connectivity graph~\citep[\eg,][]{GuptaCVPR2017}. Other forms of  privilege include
semantic maps~\citep[\eg,][]{Shridhar2020ALFREDAB,DasCoRL2018},
the ability to see into ``the future'' via rollouts~\citep{SilverNature2016}, 
and ground-truth world layouts~\citep{chen2020learning}. 
The following example shows how this type of privileged information can result in IL dramatically failing.

\begin{example}[Poisoned Doors] \label{example:poisoned-doors}
 Suppose an agent is presented with $N\geq 3$ doors $d_1,\ldots,d_N$. As illustrated in~\figref{fig:babyai-tasks-poisoned-doors} (for $N=4$), opening $d_1$ requires entering an unknown fixed code of length $M$.  Successful code entry results in a reward of $1$, otherwise the reward is $0$. Since the code is unknown to the agent, it would need to learn the code by trial and error. %
 All other doors can be opened without a code. For some randomly chosen $2 \leq j \leq N$ (sampled each episode), the reward behind $d_j$ is $2$ but for all $i\in\{2, \ldots,N\}\setminus \{j\}$ the reward behind $d_i$ is $-2$. Without knowing $j$, the optimal policy is to always enter the correct code to open $d_1$ obtaining an expected reward of $1$. In contrast, if the expert is given the privileged knowledge of the door $d_j$ with reward 2, it will always choose to open this door immediately. It is easy to see that an agent
 without knowledge of $j$ attempting to imitate such an expert will learn to open a door among $d_2,\ldots, d_N$ uniformly at random obtaining an expected return of $-2\cdot(N-3)/(N-1)$. In this setting, training with reward-based RL after a `warm start' with IL is strictly worse than starting without it: the agent  needs to unlearn its policy and then, by chance, stumble into entering the correct code for door $d_1$, a practical impossibility when $M$ is  large.
\end{example}
\begin{wrapfigure}{r}{0.3\linewidth}
    \includegraphics[width=0.3\textwidth]{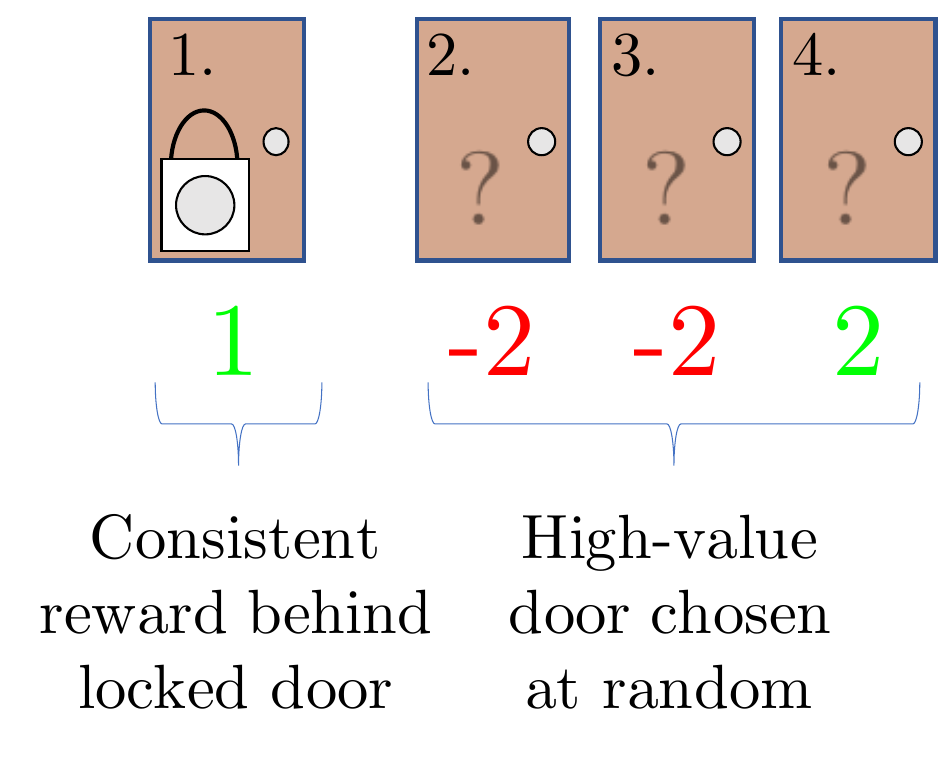}
    \caption{\textsc{PoisonedDoors}. }
    \vspace{-5mm}
    \label{fig:babyai-tasks-poisoned-doors}
\end{wrapfigure}

To characterize this imitation failure, we show that training a student to imitate a teacher who uses privileged information results in the student learning a policy which marginalizes out this privileged information.
This can result in a sub-optimal, even uniformly random, student policy over a large collection of states. We call the discrepancy between the teacher's and student's policy the \emph{imitation gap}. 
To bridge the imitation gap, we introduce 
\textbf{Ad}apti\textbf{v}e \textbf{I}n\textbf{s}ub\textbf{or}dination
(ADVISOR).
ADVISOR adaptively weights imitation and RL losses. Specifically, throughout training we use an auxiliary actor which judges whether the current observation is better treated using an IL or a RL loss. For this, the auxiliary actor attempts to reproduce the teacher's action using the observations of the student at every step. Intuitively, the weight corresponding to the IL loss is large when the auxiliary actor can reproduce the teacher's action with high confidence. %

We study the benefits of ADVISOR on thirteen tasks, including  `\textsc{PoisonedDoors}' from Ex.~\ref{example:poisoned-doors}, a 2D ``lighthouse'' gridworld, a suite of tasks set within the \minigrid environment \citep{chevalier2018babyai,gym_minigrid}, Cooperative Navigation with limited range (\cnav) in the multi-agent particle environment (MPE)~\cite{MordatchAAAI2018,Lowe2020}, and two navigational tasks set in 3D, high visual fidelity, simulators of real-world living environments (\textsc{PointNav} in \habitat \cite{habitat19iccv} and \textsc{ObjectNav} in \robothor \cite{ai2thor,DeitkeEtAl2020}). Our results show that, \\[0.5mm]
\noindent $\bullet$ the imitation gap's size directly impacts agent performance when using modern learning methods, \\[0.5mm]
\noindent $\bullet$ ADVISOR is \emph{performant} (outperforming IL and RL baselines), \emph{robust}, and \emph{sample efficient}, \\[1mm]
\noindent $\bullet$ ADVISOR can succeed even when expert supervision is partially corrupted, and \\[0.5mm]
\noindent $\bullet$ ADVISOR can be easily integrated in existing pipelines spanning diverse observations (grids and pixels), actions spaces (discrete and continuous), and algorithms (PPO and MADDPG).

\section{Related Work}
A series of methods~\citep[\eg,][]{MnihNature2015,HasseltARXIV2015,bellemare2016increasing,SchaulICLR2016} have made off-policy deep Q-learning  stable for complex environments like Atari Games. \rebuttal{Several high-performance (on-policy) policy-gradient methods for deep-RL have also been proposed~\citep{schulman2015trpo,MnihEtAlPMLR2016,levine2016end,wang2016acer,SilverNature2016}. 
For instance,} Trust Region Policy Optimization (TRPO)~\citep{schulman2015trpo} improves sample-efficiency by safely integrating larger gradient steps.
Proximal Policy Optimization (PPO)~\citep{schulman2017proximal} employs a clipped variant of TRPO's surrogate objective and 
is widely adopted in the deep RL community. We use PPO as a baseline in our experiments. 

As environments get more complex, navigating the search space with only deep RL and simple heuristic exploration (such as $\epsilon$-greedy) is increasingly difficult. Therefore,  %
methods that imitate expert (\ie, teacher) supervision %
were introduced. %
A popular approach to imitation learning (IL) is Behaviour Cloning (BC), \ie, use of a supervised classification loss between the policy of the student and expert agents~\citep{sammut1992learning,bain1995framework}. However, BC suffers from compounding errors. %
Namely, a single mistake of the student %
may lead to settings %
that have never been observed in training~\citep{ross2010efficient}. 
To address this, Data Aggregation (DAgger)~\citep{RossAISTATS2011} %
trains a sequence of student policies by querying the expert at states beyond those that would be reached by following only expert actions. 
IL is further enhanced by, \eg,  hierarchies~\citep{le2018hierarchical}, improving over the expert~\rebuttal{\citep{chang2015learning,BrysEtAl2015,JingEtAl2020}}, bypassing any intermediate reward function inference~\citep{ho2016generative}, and/or learning from experts that differ from the student~\citep{gupta2017learning,jiang2019value,gangwani2020state}.
Importantly, a sequential combination of IL and RL, \ie, pre-training a model on expert data before letting the agent interact with the environment, performs remarkably well. This strategy has been applied in a wide range of applications -- the game of Go~\citep{SilverNature2016}, robotic and motor skills~\citep{pomerleau1991efficient,kober2009policy,peters2008reinforcement,rajeswaran2018learning}, navigation
in visually realistic environments~\citep{GuptaCVPR2017,DasCVPR2018,jain2019CVPRTBONE,Jain_2021_ICCV}, and web \& language based tasks~\citep{he2016dual,visdial_rl,shi2017world,wang2018video}. 

More recent methods mix expert demonstrations with the agent's own rollouts instead of using a sequential combination of IL followed by RL. \rebuttal{\citet{ChemaliLazaric2015} perform policy iteration from expert and on-policy demonstrations.} DQfD~\citep{hester2018deep} initializes the replay buffer with expert episodes and adds rollouts of (a pretrained) agent. They weight experiences 
based on the previous temporal difference errors~\citep{SchaulICLR2016} and use a supervised loss to learn from the expert. For continuous action spaces,  DDPGfD~\citep{vecerik2017leveraging} analogously incorporates IL into DDPG~\citep{lillicrap2015continuous}.  POfD~\citep{kang2018policy} improves  by adding a demonstration-guided exploration term, \ie, the Jensen-Shannon divergence between the expert's and the learner's policy (estimated using occupancy measures). \rebuttal{THOR uses suboptimal experts to reshape rewards and then searches over a finite planning horizon \citep{SunEtAl2018}. \citet{ZhuWangEtAl2018} show that a combination of GAIL~\citep{ho2016generative} and RL can be highly effective for difficult manipulation tasks.}

\rebuttal{Critically, the above methods have, implicitly or explicitly, been designed under certain assumptions (\eg, the agent operates in an MDP) which imply the expert and student observe the same state.} Different from the above methods, we investigate the difference of privilege between the expert policy and the learned policy. Contrary to a sequential, static, or rule-based combination of supervised loss or divergence, we train an auxiliary actor to adaptively weight IL and RL losses. To the best of our knowledge, this hasn't been studied before. In concurrent work, \citet{WarringtonEtAl2020} address the imitation gap by jointly training their teacher and student to adapt the teacher to the student. For our applications of interest, this work is not applicable as our expert teachers are fixed.

\rebuttal{Our approach attempts to reduce the  imitation gap directly, assuming the information available to the learning agent is fixed. An indirect approach to reduce this gap is to enrich the information available to the agent or to improve the agent's memory of past experience. Several works have considered this direction in the context of autonomous driving~\citep{CodevillaEtAl2018ConditionalImitation,HawkeEtAl2020} and continuous control~\citep{GangwaniEtAl2019}. We expect that these methods can be beneficially combined with the method that we discuss next.}

\section{ADVISOR}

We first introduce notation to define the imitation gap and illustrate how it arises due to `policy averaging.' Using an `auxiliary policy' construct, we then propose ADVISOR to bridge this gap. Finally, we show how to estimate the auxiliary policy in practice using deep networks. In what follows we will use the terms teacher and expert interchangeably. Our use of ``teacher'' is meant to emphasize that these policies are (1) designed for providing supervision for a student and (2) need not be optimal among all policies.

\subsection{Imitation gap} \label{sec:imitation-gap}

We want an agent to complete task $\cT$  in  environment $\cE$. The environment has states $s\in\cS$ and the agent executes an action $a\in\cA$ at every discrete timestep $t\geq 0$. For simplicity and w.l.o.g.\  assume both $\cA$ and $\cS$ are finite. %
For example, let $\cE$ be a 1D-gridworld in which the agent is tasked with navigating to a location by executing actions to move left or right, as shown in Fig.~\ref{fig:lighthouse-states}. Here and below we assume states $s\in\cS$ encapsulate historical information so that %
$s$ includes the full trajectory of the agent up to  time $t\geq 0$. %
The objective is to find a policy $\pi$, \ie, a mapping from states to distributions over actions, which maximizes an evaluation criterion. Often this policy search is restricted to a set of feasible policies $\feasiblepi$, for instance $\feasiblepi$ may be the set $\{\pi(\cdot;\theta): \theta \in\bR^D\}$ where $\pi(\cdot;\theta)$ is a deep neural network with $D$-dimensional parameters $\theta$. In classical (deep) RL \citep{MnihNature2015,MnihEtAlPMLR2016}, the evaluation criterion is usually the expected $\gamma$-discounted future return. 

We focus on the setting of partially-observed Markov decision processes (POMDPs) where an agent makes decisions without access to the full state information. We model this restricted access by defining a \emph{filtration function} $f:\cS\to\cO_f$ and limiting the space of feasible policies 
to those policies $\feasiblepi_f$ %
for which the value of $\pi(s)$ depends on $s$ only through $f(s)$, 
\ie, so that $f(s)=f(s')$ implies $\pi(s)=\pi(s')$. 
We call any $\pi$ satisfying this condition an \emph{$f$-restricted policy} and the set of feasible $f$-restricted policies $\feasiblepi_f$. 
In a gridworld example, $f$ might restrict $s$ to only include information local to the agent's current position as shown in Figs.~\ref{fig:lighthouse-f1-obs}, \ref{fig:lighthouse-f2-obs}. If a $f$-restricted policy is optimal among all other $f$-restricted policies, we say it is \emph{$f$-optimal}. We call $o\in \cO_f$ a \emph{partial-observation} and for any $f$-restricted policy $\pi_f$ we %
write $\pi_f(o)$ to mean $\pi_f(s)$ if $f(s)=o$.
It is frequently the case that, during training, we have access to a \expert policy which is able to successfully complete the task $\cT$. This \expert policy may have access to the whole environment state and thus may be optimal among all policies. Alternatively, the \expert policy may, like the student, only  make decisions given partial information (\eg, a human who sees exactly the same inputs as the student).
For flexibility we will define the \expert policy as $\piexp_{\fexp}$, denoting it is an $\fexp$-restricted policy for some filtration function $\fexp$. For simplicity, we will assume that $\piexp_{\fexp}$ is $\fexp$-optimal. Subsequently, we will drop the subscript $\fexp$ unless we wish to explicitly discuss multiple \expert{}s simultaneously.

In IL \citep{OsaPajarinenEtAl2018,RossAISTATS2011}, $\pi_f$ is trained to mimic $\piexp$ by minimizing the (expected) cross-entropy between $\pi_f$ and $\piexp$ over a set of sampled states $s\in\cS$: %
\begin{align}
    \min_{\pi_f\in\feasiblepi_f} \mathbb{E}_\mu[CE(\piexp, \pi_f)(S)]\;, \label{eq:basic_imitation}
\end{align}
where $CE(\piexp, \pi_f)(S) = -\piexp(S)\odot \log\pi_f(S)$, $\odot$ denotes the usual dot-product, and $S$ is a random variable taking values $s\in\cS$ with probability measure $\mu:\cS\to[0,1]$. Often $\mu(s)$ is chosen to equal the frequency with which an exploration policy
(\eg, random actions or $\piexp$) visits state $s$ in a randomly initialized episode. When it exists, we denote the policy minimizing Eq.~\eqref{eq:basic_imitation} as $\pi^{\mu, \piexp}_{f}$. When $\mu$ and $\piexp$ are unambiguous, we write $\pi^{\text{IL}}_{f}=\pi^{\mu, \piexp}_{f}$. %

What happens when there is a difference of privilege (or filtration functions) between the \expert and the student?
Intuitively, if the information that a \expert uses to make a decision is unavailable to the student then the student has little hope of being able to mimic the \expert{}'s decisions. As we show in our next example, even when optimizing perfectly, depending on the choice of $f$ and $\fexp$, IL may result in $\pi^{\text{IL}}_{f}$ being uniformly random over a large collection of states. We call the phenomenon that $\pi^{\text{IL}}_f\not= \piexp$ the \emph{imitation gap}. %

\begin{figure}
    \centering
    {
      \phantomsubcaption\label{fig:lighthouse-states}
      \phantomsubcaption\label{fig:lighthouse-ep-example}
      \phantomsubcaption\label{fig:lighthouse-f1-obs}
      \phantomsubcaption\label{fig:lighthouse-f2-obs}
    }
    \includegraphics[width=\textwidth]{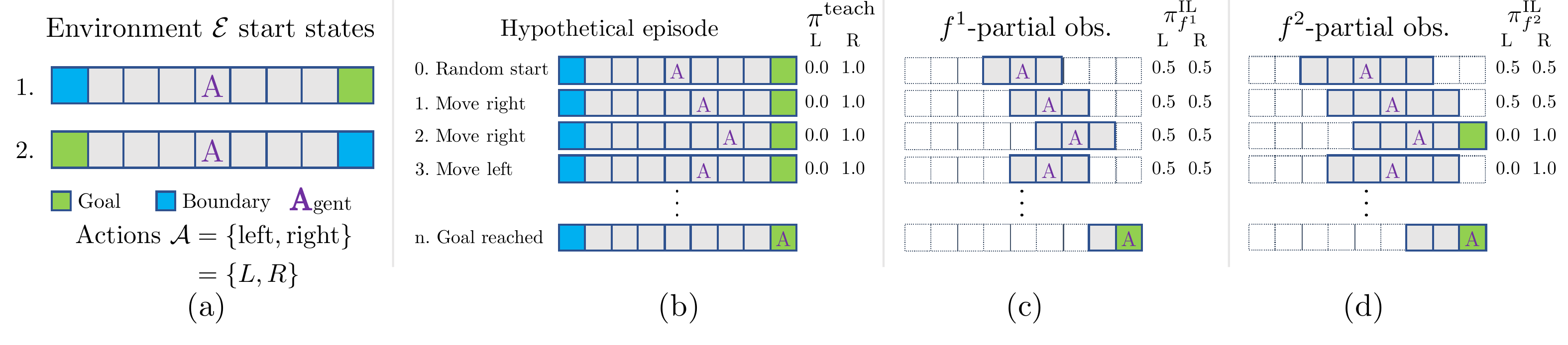}
    \caption{\textbf{Effect of partial observability in a 1-dimensional gridworld environment.} (a) The two start states and actions space for 1D-Lighthouse with $N=4$. (b) A trajectory of the agent following a hypothetical random policy. At every trajectory step we display output probabilities as per the shortest-path expert ($\piexp$) for each state. (c/d) Using the same trajectory from (b) we highlight the partial-observations available to the agent (shaded gray) under different filtration function $f^1, f^2$. Notice that, under $f^1$, the agent does not see the goal within its first four steps. The policies $\pi^{\text{IL}}_{f^1},\pi^{\text{IL}}_{f^2}$, learned by imitating $\piexp$, show that imitation results in sub-optimal policies, \ie, $\pi^{\text{IL}}_{f^1},\pi^{\text{IL}}_{f^2} \neq \piexp$.
}
    \label{fig:imitation-gap}
\end{figure}

\begin{example}[1D-Lighthouse] \label{ex:lighthouse} 
\rebuttal{We illustrate the imitation gap using a  gridworld spanning $\{-N,\dots, N\}$. The two start states correspond to the goal being at either $-N$ or $N$, while the agent is always initialized at $0$ (see Fig.~\ref{fig:lighthouse-states}). Clearly, with full state information, $\piexp$ maps states to an `always left' or `always right' probability distribution, depending on whether the goal is on the left or right, respectively. 
Suppose now that the agent's visibility is constrained to a radius of $i$ (Fig.~\ref{fig:lighthouse-f1-obs} shows $i=1$), \ie, an $f^i$-restricted observation is accessible. An agent following an optimal policy with a visibility of radius $i$ will begin to move deterministically towards any corner, w.l.o.g.\ assume right. When the agent sees the rightmost edge (from position $N-i$), it will either continue to move right if the goal is visible or, if it's not, move left until it reaches the goal (at $-N$). 
Now we may ask: what is the best $f^{i}$-restricted policy that can be learnt by imitating $\piexp$ (\ie, what is $\pi_{f^i}^{\text{IL}}$)? \emph{Tragically, the cross-entropy loss causes $\pi_{f^i}^{\text{IL}}$ to be uniform in a large number of states.} In particular, an agent following policy $\pi_{f^i}^{\text{IL}}$ will execute left (and right) actions with probability $0.5$, until it is within a distance of $i$ from one of the corners. Subsequently, it will head directly to the goal. See the policies highlighted in Figs.~\ref{fig:lighthouse-f1-obs}, \ref{fig:lighthouse-f2-obs}. The intuition for this result is straightforward: until the agent observes one of the corners it cannot know if the goal is to the right or left and, conditional on its observations, each of these events is equally likely under $\mu$ (assumed uniform). Hence for half of these events the \expert will instruct the agent to go right. For the other half the instruction is to go left. See App.~\ref{sec:formal-lighthouse-example} for a rigorous treatment of this example. In Sec.~\ref{sec:experiments} and Fig.~\ref{fig:lh-plots}, we train $f^i$-restricted policies with $f^j$-optimal \expert{}s for a 2D variant of this example. We empirically verify that a student learns a better policy when imitating \expert{}s whose filtration function is closest to their own.}

\end{example}

The above example shows: when a student attempts to imitate an expert that is privileged with information not available to the student, the student learns a version of $\piexp$ in which this privileged information is marginalized out. We formalize this intuition in the following proposition.

\begin{prop}[Policy Averaging]\label{prop:policy-averaging}~\\
In the setting of Section \ref{sec:imitation-gap}, suppose that $\feasiblepi$ contains %
all $f$-restricted policies. Then, for any $s\in\cS$ with $o=f(s)$, we have that $\pi^{\text{IL}}_f(o) = \mathbb{E}_{\mu}[\piexp(S)\mid f(S) = o].$
\end{prop}
Given our definitions, the proof of this proposition is quite straightforward, see Appendix~\ref{sec:policy-averaging-proof}.

The imitation gap provides  theoretical justification for the common practical observation %
that an agent trained via  IL  can often be significantly improved by continuing to train the agent using pure RL (\eg, PPO) \citep{Lowe2020,DasCoRL2018}. Obviously training first with IL and then via pure RL %
is ad~hoc and potentially sub-optimal as discussed in Ex.~\ref{example:poisoned-doors} and empirically shown in Sec.~\ref{sec:experiments}. 
To alleviate this problem, the student should 
imitate the \expert{}'s policy only in settings where the \expert{}'s policy can, in principle, be exactly reproduced by the student. Otherwise the student should learn via `standard' RL. To achieve this we introduce ADVISOR.

\subsection{Adaptive Insubordination (ADVISOR) with Policy Gradients} \label{sec:adaptive-method}

To close the imitation gap, ADVISOR adaptively weights reward-based and imitation losses. Intuitively, it supervises a student by asking it to imitate a \expert{}'s policy only in those states $s\in\cS$ for which the imitation gap is small. For all other states, it trains the student using reward-based RL. 
To simplify notation, we denote the reward-based RL loss via  $\mathbb{E}_\mu[L(\theta, S)]$ for some loss function $L$.\footnote{
For readability, we  implicitly make three key simplifications. First, computing the expectation $\mathbb{E}_\mu[\ldots]$ is generally intractable, hence we cannot directly minimize losses such as $\mathbb{E}_\mu[L(\theta, S)]$. Instead, we approximate the expectation using rollouts from $\mu$ and optimize the empirical loss. 
Second, recent RL methods adjust the measure $\mu$ over states as optimization progresses while we assume it to be static for simplicity. Our final simplification regards the degree to which any loss can be, and is, optimized. In general, losses are often optimized by gradient descent and generally no guarantees are given that the global optimum can be found. Extending our presentation to encompass these issues is straightforward but notationally dense. 
}
This loss formulation is general and spans all policy gradient methods, including A2C and PPO. The imitation loss is the standard cross-entropy loss $\mathbb{E}_\mu[CE(\piexp(S), \pi_f(S;\theta))]$. Concretely, the ADVISOR loss is:
\begin{align}
    \cL^{\text{ADV}}(\theta) = \mathbb{E}_\mu[w(S) \cdot CE(\piexp(S), \pi_f(S;\theta)) + (1- w(S))\cdot L(\theta, S)] \; .\label{eq:advisor-loss}
\end{align}
Our goal is to find a \emph{weight function} $w:\cS \to [0,1]$ where $w(s)\approx 1$ when the imitation gap is small and $w(s)\approx 0$ otherwise.
For this we need an estimator of the distance between $\piexp$ and $\pi^{\text{IL}}_f$ at a state $s$ and a mapping from this distance to weights in $[0,1]$. %

We now define  a distance estimate $d^{0}(\pi, \pi_f)(s)$ between a policy $\pi$ and an $f$-restricted policy $\pi_f$ at a state $s$. We can  use  any common non-negative distance (or divergence) $d$ between probability distributions on $\cA$, \eg, in our experiments we use the  KL-divergence. While there are many possible strategies for using $d$ to estimate $d^{0}(\pi, \pi_f)(s)$, perhaps the simplest of these strategies is to define $d^{0}(\pi, \pi_f)(s) = d(\pi(s), \pi_f(s))$. Note that this quantity does not attempt to use any information about the fiber $f^{-1}(f(s))$ which may be useful in producing more holistic measures of distances.\footnote{Measures using such information include $\max_{s'\in f^{-1}(f(s)}d(\pi(s'), \pi_f(s))$ or a corresponding expectation instead of the maximization, \ie, $\mathbb{E}_{\mu}[d(\pi(S), \pi_f(S)) \mid f(S) = o]$.} Appendix~\ref{sec:other-distances} considers how those distances can be used in lieu of $d^{0}$. Next, using the above, we need to estimate the quantity $d^{0}(\piexp,\pi^{\text{IL}}_f)(s)$. 

Unfortunately it is, in general, impossible to compute $d^{0}(\piexp, \pi^{\text{IL}}_f)(s)$ exactly as it is intractable to compute the optimal minimizer $\pi^{\text{IL}}_{f}$. Instead we leverage an estimator of $\pi^{\text{IL}}_{f}$ which we term $\pi^{\text{aux}}_{f}$, and which we will define in the next section.

Given $\pi^{\text{aux}}_f$ we obtain the estimator $d^{0}(\piexp, \pi^{\text{aux}}_{f})$ of $d^{0}(\piexp, \pi^{\text{IL}}_{f})$. Additionally, we make use of the monotonically decreasing function $m_{\alpha}:\bR_{\geq 0} \to[0,1]$, where $\alpha \geq 0$. %
We define our weight function $w(s)$ for $s\in\cS$ as: 
\begin{align}
    w(s) &= m_{\alpha}(d^{0}(\piexp, \pi^{\text{aux}}_{f})(s))\quad \text{with}\quad
    m_{\alpha}(x) = e^{-\alpha x}. \label{eq:weight}
\end{align}

\subsection{The Auxiliary Policy $\pi^{\text{aux}}$: Estimating $\pi^{\text{IL}}_f$ in Practice} \label{sec:adaptive-in-practice}

In this section we describe how we can, during training, obtain an \emph{auxiliary policy} $\pi^{\text{aux}}_{f}$ which estimates $\pi^{\text{IL}}_{f}$. Given this auxiliary policy we  estimate $d^{0}(\piexp, \pi^{\text{IL}}_{f})(s)$ using the plug-in estimator $d^{0}(\piexp, \pi^{\text{aux}}_{f})(s)$. While plug-in estimators are intuitive and simple to define, they need not be statistically efficient. 
In Appendix \ref{sec:future-work-improving-plugins} we consider possible strategies for improving the statistical efficiency of our plug-in estimator via prospective estimation.
\begin{wrapfigure}{r}{0.5\linewidth}
    \includegraphics[width=1.0\linewidth]{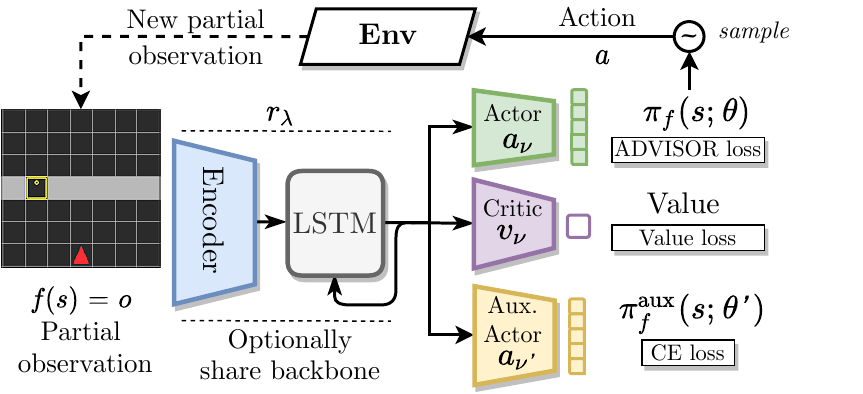}
    \caption{\textbf{Model overview.} An auxiliary actor is added and trained only using IL. The `main' actor policy is trained using the ADVISOR loss.
    }
    \label{fig:model}
    \vspace{-5mm}
\end{wrapfigure}

In \figref{fig:model} we provide an overview of how we compute the estimator $\pi^{\text{aux}}_{f}$ via deep nets. As is common practice~\citep{MnihEtAlPMLR2016,heess2017emergence,jaderberg2017reinforcement,pathak17curiosity,mirowski2017learning,chevalier2018babyai,chen2019audio,JainWeihs2020CordialSync,wani2020multion}, the policy net $\pi_f(\cdot;\theta)$ is composed via $a_\nu\circ r_\lambda$ with $\theta = (\nu, \lambda)$, where $a_\nu$  is the \emph{actor head} (possibly complemented in actor-critic models by a \emph{critic head} $v_\nu$) and $r_\lambda$ is called the \emph{representation network}. 
Generally $a_\nu$ is lightweight, for instance a linear layer or a shallow MLP followed by a soft-max function, while $r_\lambda$ is a deep, and possibly recurrent neural, net. We add
another actor head $a_{\nu'}$ to our existing network which shares the underlying representation $r_\lambda$, \ie, $\pi^{\text{aux}}_f = a_{\nu'} \circ r_\lambda$. We experiment with the actors sharing their representation $r_\lambda$ and estimating $\pi^{\text{IL}}_f$ via two separate networks, \ie, $\theta' = (\nu', \lambda')$. 
\rebuttal{In practice we train $\pi_f(\cdot;\theta)$ and $\pi^{\text{aux}}_{f}(\cdot;\theta)$ jointly using stochastic gradient descent, as summarized in Alg. \ref{alg:onpolicy-advisor}.}

\section{Experiments}\label{sec:experiments}

\begin{figure}
    \centering
    \includegraphics[width=\linewidth]{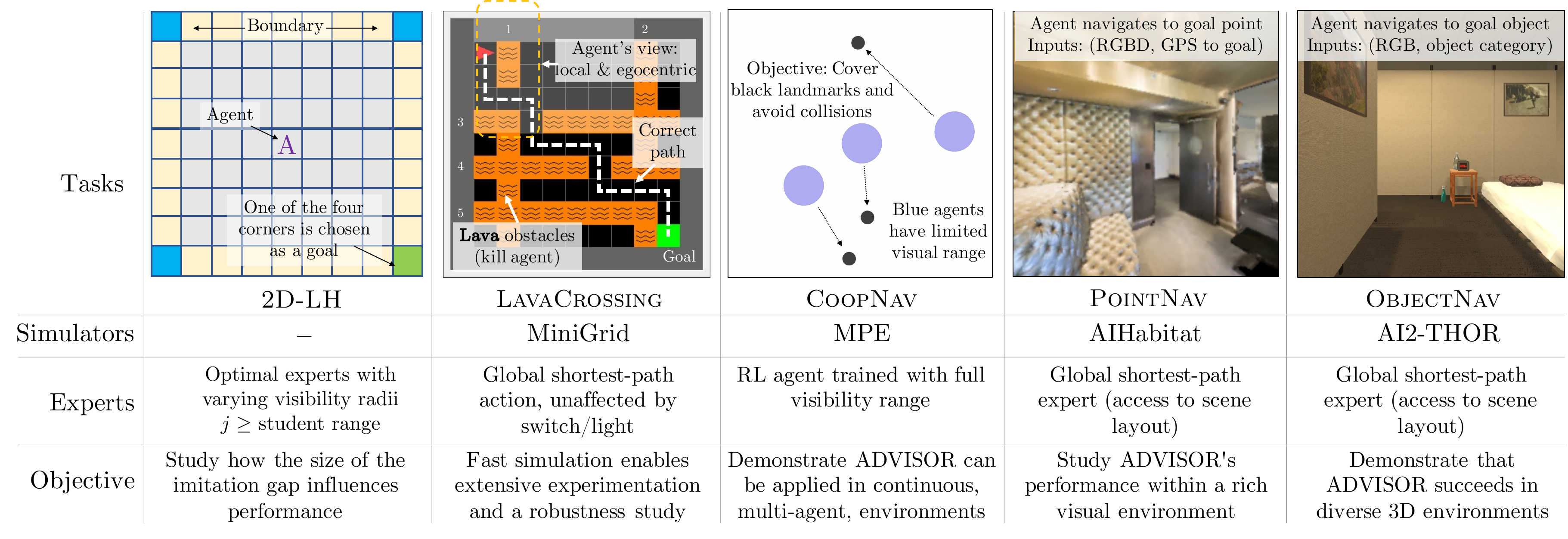}
    \caption{\textbf{Representative tasks from experiments.}
    \twodimlh: Harder 2D variant of the gridworld task introduced in Ex.~\ref{ex:lighthouse}.
    \textsc{LavaCrossing}: one of our 8 tasks in the \minigrid environment requiring safe navigation. We test up-to $15\times15$ grids with $10$ lava rivers.
    \cnav: A multi-agent cooperative task set in multi-agent particle environments.
    \pnav: An agent embodied in the \habitat environment must navigate using egocentric visual observations to a goal position specified by a GPS coordinate.
    \onav: An agent in \robothor must navigate to an object of a given category.
    }
    \label{fig:tasks}
\end{figure}

We rigorously compare ADVISOR %
to IL methods, RL methods, and popularly-adopted (but often ad hoc) IL \& RL combinations. In particular, we evaluate $15$ learning methods. We do this over thirteen tasks -- realizations of Ex.~\ref{example:poisoned-doors} \& Ex.~\ref{ex:lighthouse}, eight tasks of varying complexity within the fast, versatile {\minigrid} environment~\citep{chevalier2018babyai,gym_minigrid}, Cooperative Navigation (\cnav) with reduced visible range in the multi-agent particle environment (\text{MPE}) \citep{MordatchAAAI2018,LoweNIPS2017}, PointGoal navigation (\pnav) using the Gibson dataset in \habitat \citep{xia2018gibson,habitat19iccv}, and ObjectGoal Navigation (\onav) in \robothor\citep{DeitkeEtAl2020}.\footnote{The \robothor environment is a sub-environment of \thor\cite{ai2thor}.} Furthermore, to probe robustness, we train $50$ hyperparameter variants for each of the 15 learning methods for our \minigrid tasks. We find ADVISOR-based methods outperform or match performance of all baselines.

All code to reproduce our experiments will be made public under the Apache 2.0 license.\footnote{See \url{https://unnat.github.io/advisor/} for an up-to-date link to this code.} The environments used are public for academic and commercial use under the Apache 2.0 (\minigrid and \robothor) and MIT licence (MPE and \habitat).

\subsection{Tasks}\label{sec:tasks}
Detailed descriptions of our tasks (and teachers) are deferred to Appendix~\ref{sec:additional-task-details}. See Fig.~\ref{fig:tasks} for a high-level overview of 5 representative tasks.

\subsection{Baselines and ADVISOR-based Methods}\label{sec:baselines} 
We briefly introduce baselines and variants of our ADVISOR method. Further details of all methods are in Appendix~\ref{sec:additional-baseline-details}. For fairness, the same model architecture is shared across all methods (recall \figref{fig:model},~\secref{sec:adaptive-in-practice}). We defer implementation details  to  Appendix~\ref{sec:arch-details}.

$\bullet$~\textbf{RL only.} Proximal Policy Optimization~\cite{schulman2017proximal} serves as the pure RL baseline for all our tasks with a discrete action space. For the continuous and multi-agent \cnav task, we follow prior work and adopt MADDPG~\cite{LoweNIPS2017,LiuCORL2019}.\\[1mm]
$\bullet$~\textbf{IL only.} IL baselines where supervision comes from an expert policy with different levels of teacher-forcing (tf), \ie, tf=0, tf annealed from 1$\rightarrow$0, and tf=1. This leads to Behaviour Cloning (BC), Data Aggregation (DAgger or ~$\dagger$), and $\text{BC}^{\text{tf=1}}$, respectively~\citep{sammut1992learning,bain1995framework,RossAISTATS2011}.\\[1mm]
$\bullet$~\textbf{IL \& RL.} Baselines that use a mix of IL and RL losses, either in sequence or in parallel. These are popularly adopted in the literature to warm-start agent policies. Sequential combinations include BC then PPO ({\bcppo}), DAgger then PPO ($\dagger\rightarrow\text{PPO}$), and $\text{BC}^{\text{tf=1}}\rightarrow \text{PPO}$. The parallel combination of $\text{BC}+\text{PPO} (\text{static})$ is a static analog of our adaptive combination of IL and RL losses.\\[1mm]
$\bullet$~\textbf{Demonstration-based.} These agents imitate expert demonstrations and hence get no supervision beyond the states in the demonstrations. We implement $\text{BC}^{\text{demo}}$, its combination with PPO ($\text{BC}^{\text{demo}}+\text{PPO}$), and Generative Adversarial Imitation Learning (GAIL)~\cite{ho2016generative}.\\[1mm]
$\bullet$~\textbf{ADVISOR-based (ours).} Our Adaptive Insubordination methodology can learn from an expert policy and can be given a warm-start via BC or DAgger. This leads to ADVISOR (ADV), $\text{BC}^{\text{tf=1}}\rightarrow\text{ADV}$, and $\dagger\rightarrow\text{ADV}$) baselines. Similarly, $\text{ADV}^{\text{demo}}+\text{PPO}$ employs Adaptive Insubordination to learn from expert demonstrations while training with PPO on on-policy rollouts.\\

\subsection{Evaluation}\label{sec:evaluation}

\begin{figure}
    \centering
    {
    \phantomsubcaption\label{fig:poisoned-doors-eval-train-steps}
      \phantomsubcaption\label{fig:lava-cross-eval-train-steps}
      \phantomsubcaption\label{fig:lava-cross-switch-eval-train-steps}
      \phantomsubcaption\label{fig:lava-cross-corrupt-eval-train-steps}
      \phantomsubcaption\label{fig:poisoned-doors-eval}
      \phantomsubcaption\label{fig:lava-cross-eval}
      \phantomsubcaption\label{fig:lava-cross-switch-eval}
      \phantomsubcaption\label{fig:lava-cross-corrupt-eval}
    }
    \includegraphics[width=1\linewidth]{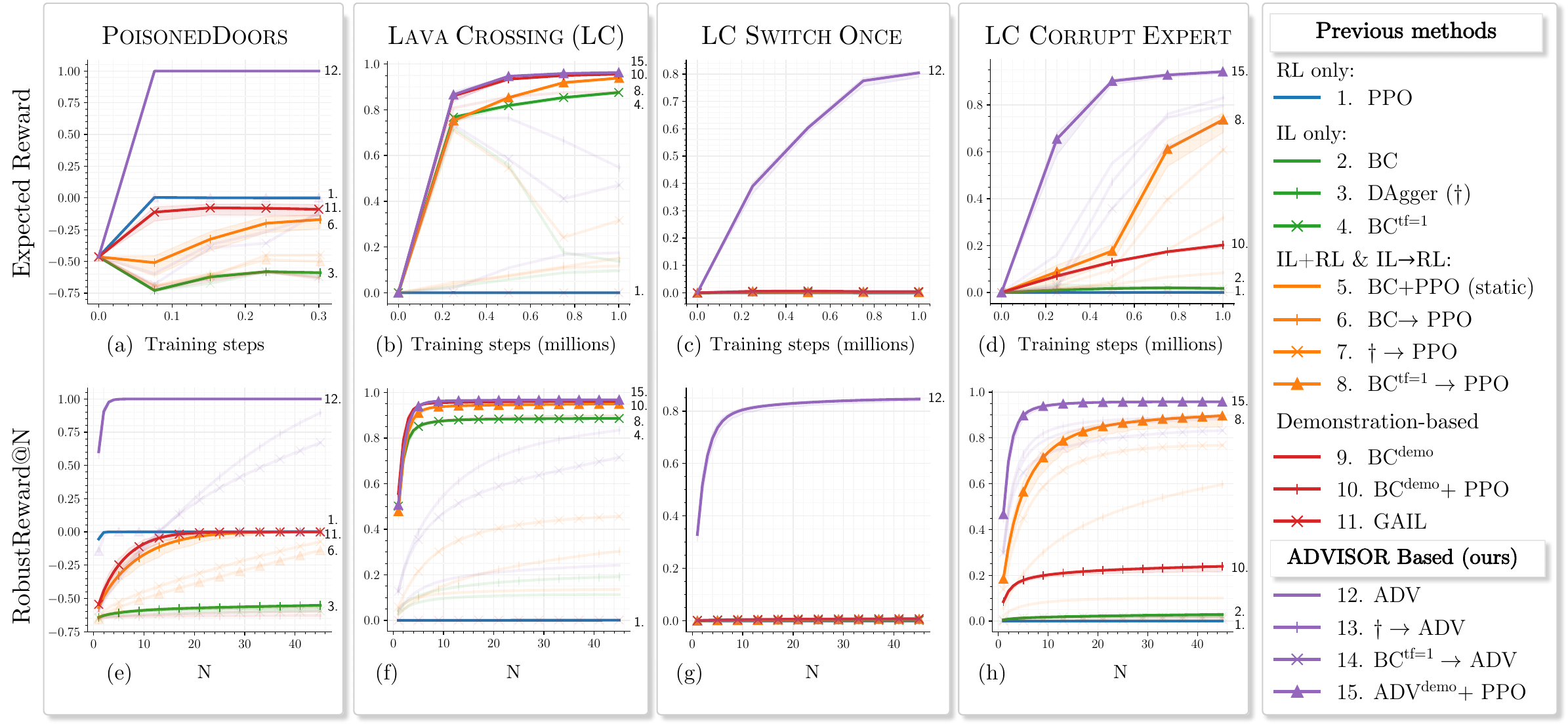}
    \caption{\textbf{Evaluation following \citep{DodgeGCSS19}.} Plots for 15 training routines in four selected tasks (additional plots in appendix). For clarity, we highlight the best performing training routine within five categories, \eg, RL only, IL only \etc (details in~\secref{sec:baselines}) with all other plots shaded lighter. (a)-(d) As described in~\secref{sec:evaluation} we plot \emph{RobustReward@$10$} at multiple points during training. (e)-(f) Plots of \emph{RobustReward@$N$} for values of $N\in\{1,\dots,45\}$. Recall that \emph{RobustReward@$N$} is the expected validation reward of best-found model when allowed $N$ random hyperparameter evaluations.}
    \label{fig:minigrid-tasks}
\end{figure}

\begin{table}[t]
\centering
\resizebox{\textwidth}{!}{%
\begin{tabular}{l>{\columncolor{ColorPNav}}c>{\columncolor{ColorFurnMove}}c>{\columncolor{ColorFurnMove}}c>{\columncolor{ColorFurnMove}}c>{\columncolor{ColorFurnMove}}c>{\columncolor{ColorFootball}}c>{\columncolor{ColorFootball}}c>{\columncolor{ColorFootball}}c>{\columncolor{ColorFootball}}c>{\columncolor{ColorFootball}}c>{\columncolor{ColorFootball}}cc}
\toprule
\multicolumn{1}{c}{Tasks $\rightarrow$} & \textbf{\textsc{PD}} & \multicolumn{4}{c}{\cellcolor{ColorFurnMove}\textbf{\textsc{LavaCrossing}}} & \multicolumn{4}{c}{\cellcolor{ColorFootball}\textbf{\textsc{WallCrossing}}}  \\
 Training routines $\downarrow$  & - & Base Ver. & Corrupt Exp. & Faulty Switch & Once Switch & Base Ver. & Corrupt Exp. & Faulty Switch & Once Switch  \\ \midrule
RL & 0 & 0 & 0 & 0.01 & 0 & 0.09 & 0.07 & 0.12 & 0.05 \\
IL & -0.59 & 0.88 & 0.02 & 0.02 & 0 & 0.96 & 0.05 & 0.17 & 0.11 \\
IL \& RL & -0.17 & 0.94 & 0.74 & 0.04 & 0 & \textbf{0.97} & 0.18 & 0.17 & 0.1 \\
Demo. Based & -0.09 & \textbf{0.96} & 0.2 & 0.02 & 0 & \textbf{0.97} & 0.07 & 0.18 & 0.11 \\
ADV. Based (ours) & \textbf{1} & \textbf{0.96} & \textbf{0.94} & \textbf{0.77} & \textbf{0.8} & \textbf{0.97} & \textbf{0.31} & \textbf{0.38} & \textbf{0.45}\\ \bottomrule
\end{tabular}
}%
~\vspace{3mm}
\caption{\textbf{Expected rewards for the \textsc{PoisonedDoors} task and \textsc{MiniGrid} tasks.} For each of our 15 training routines we report the expected maximum validation set performance (when given a budget of 10 random hyperparameter evaluations) after training for $\approx$300k steps in \textsc{PoisonedDoors} and $\approx$1Mn steps in our 8 \minigrid tasks. The maximum reward is 1 for the \minigrid tasks.} \label{tab:pd-and-mingrid-results}
\end{table}

\begin{table}[t]
\centering
\resizebox{\textwidth}{!}{%
\begin{tabular}{l>{\columncolor{ColorPNav}}c>{\columncolor{ColorPNav}}c>{\columncolor{ColorPNav}}c>{\columncolor{ColorPNav}}c>{\columncolor{ColorFurnMove}}c>{\columncolor{ColorFurnMove}}c>{\columncolor{ColorFurnMove}}c>{\columncolor{ColorFurnMove}}c>{\columncolor{ColorFootball}}c>{\columncolor{ColorFootball}}c}

\toprule
\multicolumn{1}{c}{Tasks $\rightarrow$} & \multicolumn{4}{c}{\cellcolor{ColorPNav}\textbf{PointGoal Navigation}} & \multicolumn{4}{c}{\cellcolor{ColorFurnMove}\textbf{ObjectGoal Navigation}} & \multicolumn{2}{c}{\cellcolor{ColorFootball}\textbf{Cooperative Navigation}} \\
\multicolumn{1}{c}{} & \multicolumn{2}{c}{\cellcolor{ColorPNav}SPL} & \multicolumn{2}{c}{\cellcolor{ColorPNav}Success} & \multicolumn{2}{c}{\cellcolor{ColorFurnMove}SPL} & \multicolumn{2}{c}{\cellcolor{ColorFurnMove}Success} & \multicolumn{2}{c}{\cellcolor{ColorFootball}Reward} \\
Training routines $\downarrow$ & \emph{@10\%} & \emph{@100\%} & \emph{@10\%} & \emph{@100\%} & \emph{@10\%} & \emph{@100\%} & \emph{@10\%} & \emph{@100\%} & \emph{@10\%} & \emph{@100\%} \\
\midrule
RL only
& 30.9 & 54.7 & 54.7 & 79.0 &
6.7 & 13.1 & 11.1 & \textbf{31.6} &
$-$561.8 & $-$456.0 \\
IL only 
& 30.1 & 68.7 & 35.5 & 76.7 &
3.8 & 9 & 8.8 & 13.6 &
$-$460.3 & $-$416.7 \\
IL + RL static &
48.9 & 71.5 & 56.7 & 78.2 &
6.5 & 11.3 & 11.7 & 19.8 &
$-$475.5 & $-$424.6 \\
ADVISOR (ours) &
\textbf{57.7} & \textbf{77.1} & \textbf{67.3} & \textbf{88.2} &
\textbf{11.9} & \textbf{14.1} & \textbf{22.7} & 29.9 &
\textbf{$-$419.9} & \textbf{$-$405.6} \\
\bottomrule
\end{tabular}%
} %
\vspace{0.03in}
\caption{\textbf{Quantitative results for high-fidelity visual environments and continuous control.} Validation set performance after 10\% and 100\% of training has completed for four training routines on the \pnav, \onav, and \cnav tasks (specifics of these routines can be found in the Appendix). For \pnav and \onav we include the common success weighted path length (SPL) metric \cite{anderson2018evaluation} in addition to the success rate.  
}
\label{tab:complex-tasks}
\vspace{-4mm}
\end{table}

\noindent\textbf{Fair Hyperparameter Tuning.} Often unintentionally done, extensively tuning the hyperparameters ({hps}) of a proposed method and not those of the baselines can introduce unfair bias into evaluations. We avoid this by considering two strategies. 
For \pd and all \textsc{MiniGrid} tasks, we follow recent best practices \citep{DodgeGCSS19}.
Namely, we tune each method by randomly sampling a fixed number of {hps} and reporting, for each baseline, an estimate of 
\begin{align}
\text{\emph{RobustReward@$K$}}=\mathbb{E}[\text{Val. reward of best model from $k$ random hyperparam. evaluations}]
\end{align}
for $1\leq k\leq 45$. For this we must train 50 models per method, \ie, 750 for each of these nine tasks. In order to show learning curves over training steps we also report \emph{RobustReward@$10$} at 5 points during training. More details in Appendix~\ref{sec:hyperparameter-tuning}. For \twodimlh, we tune the {hps} of a competing method and use these {hps} for all other methods.\\
\noindent\textbf{Training.} For the eight \textsc{MiniGrid} tasks, we train each of the $50$ training runs for 1 million steps. For \twodimlh/\pd, models saturate much before $3\cdot 10^5$ steps. \pnav, \onav, and \cnav are trained for standard budgets of 50Mn, 100Mn, and 1.5Mn steps. Details are in Appendix~\ref{sec:training-imp-details}.\\
\noindent\textbf{Metrics.} We record standard metrics for each task. This includes avg.\  rewards (\pd, \minigrid tasks, and \onav), and avg.\ episode lengths (\twodimlh). Following visual navigation works~\cite{anderson2018evaluation,habitat19iccv,DeitkeEtAl2020}, we report success rates and success-weighted path length (SPL) for \pnav and \onav. In the following, we report a subset of the above and defer additional plots to Appendix~\ref{sec:more-plots}.

\subsection{Results}\label{sec:results}

In the following, we include takeaways based on the results in \figref{fig:minigrid-tasks}, \figref{fig:lh-plots}, \tabref{tab:pd-and-mingrid-results}, and \tabref{tab:complex-tasks}.

\noindent\textbf{Smaller imitation gap $\implies$ better performance.} A central claim of our paper is that the imitation gap is not merely a theoretical concern: the degree to which the teacher is privileged over the student has significant impact on the student's performance.
To study this empirically, we vary the degree to which teachers are privileged over its students in our \twodimlh task. In particular, we use behavior cloning to train an $f^i$-restricted policy (\ie, an agent that can see $i$ grid locations away) using an $f^j$-optimal teacher 25 times. Each policy is then evaluated on 200 random episodes and the average episode length (lower being better) is recorded. For select $i,j$ pairs we show boxplots of the 25 average episode lengths in Fig.~\ref{fig:lh-plots}. See our appendix for similar plots when using other training routines (\eg, ADVISOR).
\begin{wrapfigure}{r}{0.5\linewidth}
    \vspace{-4mm}
    \includegraphics[width=1\linewidth]{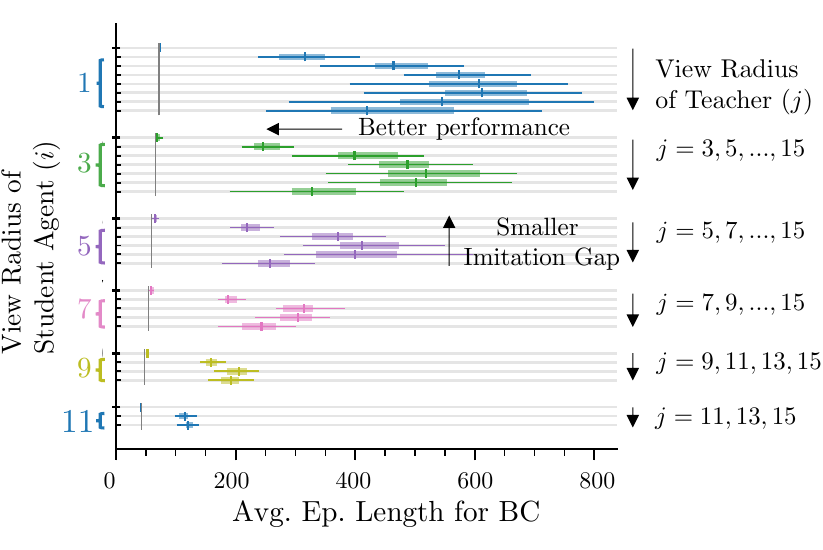}
    \vspace{-7mm}
\caption{The size of the imitation gap directly impacts performance (in  \twodimlh).
    }
    \label{fig:lh-plots}
    \vspace{-6mm}
\end{wrapfigure}
Grey vertical lines show optimal average episode lengths for $f^i$-restricted policies. We find that training an $f^i$-restricted policy with an $f^j$-expert results in a near optimal policy when $i=j$ but even small increases in $j$ dramatically decrease performance. While performance tends to drop with increasing $j$, the largest $i,j$ gaps do not consistently correspond to the worst performing models. While this seems to differ from our results in Ex.~\ref{ex:lighthouse}, recall that there the policy $\mu$ was fixed while here it varies through training, resulting in complex learning dynamics. Surprisingly we also find that, even when there is no imitation gap (\eg, the $i=j$ case), ADVISOR can outperform BC, see App.~\ref{sec:lighthouse-efficiency-study}.

\noindent\textbf{ADVISOR outperforms, even in complex visual environments.} Across all of our tasks, ADVISOR-based methods perform as well or better than competing methods. In particular, see Tab.~\ref{tab:pd-and-mingrid-results} for our results on the \textsc{PoisonedDoors} (\pd) and \minigrid tasks
and Tab.~\ref{tab:complex-tasks} for our results on the \pnav, \onav, and \cnav tasks. \twodimlh results are deferred to the Appendix.

While the strong performance of ADVISOR is likely expected on our \pd, \minigrid, and \twodimlh tasks (indeed we designed a subset of these  with the explicit purpose of studying the imitation gap), it is nonetheless surprising to see that in the \pd and \textsc{LC Once Switch} tasks, all non-ADVISOR methods completely fail. Moreover, it is extremely promising to see that ADVISOR can provide substantial benefits in a variety of standard tasks, namely \onav, \pnav, and \cnav with limited visible range. Note that \onav and \pnav are set in 3D high-fidelity visual environments while \cnav requires multi-agent collaboration in a continuous space.

\noindent\textbf{ADVISOR is sample efficient.} To understand the sample efficiency of ADVISOR, we plot validation set performance  over training of select tasks (see \Cref{fig:poisoned-doors-eval-train-steps,fig:lava-cross-eval-train-steps,fig:lava-cross-switch-eval-train-steps,fig:lava-cross-corrupt-eval-train-steps}) and, in Table \ref{tab:complex-tasks} we show performance of our models after 10\% of training has elapsed for the \onav, \pnav, and \cnav tasks. Note that in Table \ref{tab:complex-tasks}, ADVISOR trained models frequently reach better performance after 10\% of training than other methods manage to reach by the end of training.

\noindent\textbf{ADVISOR is robust.} Rigorously studying sensitivity to hyperparameter choice requires retraining every method under consideration tens to hundreds of times. This computational task can make evaluating our methods on certain tasks infeasible (training a single \pnav or \onav model can easily require a GPU-week of computation). Because of these computational constraints, we limit our study of robustness to the \pd and \minigrid tasks. In \Cref{fig:poisoned-doors-eval,fig:lava-cross-eval,fig:lava-cross-switch-eval,fig:lava-cross-corrupt-eval} (additional results in Appendix) we plot, for each of the 15 evaluated methods, how the expected performance of each method behaves as we increase the budget of random hyperparameter evaluations. In general, relatively few hyperparameter evaluations are required for ADVISOR before a high performance model is expected to be found.

\noindent\textbf{Expert demonstrations can be critical to success.} While it is frequently assumed that on-policy expert supervision is better than learning from off-policy demonstrations, we found several instances in our \minigrid experiments where demonstration-based methods outperformed competing methods. See, for example,  \Cref{fig:lava-cross-eval-train-steps,fig:lava-cross-eval}. In such cases our demonstration-based ADVISOR variant (see Appendix~\ref{sec:additional-baseline-details} for details) performed very well. 

\noindent\textbf{ADVISOR helps even when the expert is corrupted.} In \textsc{LC Corrupt Expert} and \textsc{WC Corrupt Expert}, where the expert is designed to be corrupted (outputting random actions as supervision) when the agent gets sufficiently close to the goal. While ADVISOR was not designed with the possibility of corrupted experts in mind, \Cref{fig:lava-cross-corrupt-eval-train-steps,fig:lava-cross-corrupt-eval} (see also Table \ref{tab:pd-and-mingrid-results}) show that ADVISOR can succeed despite this corruption.

\vspace{-2mm}
\section{Conclusion}
\vspace{-2mm}
We propose the \emph{imitation gap} as one explanation for the empirical observation that imitating ``more intelligent'' teachers can lead to worse policies. While prior work has, implicitly, attempted to bridge this gap,
we introduce a principled adaptive weighting technique (ADVISOR), which we test on a suite of thirteen tasks. Due to the fast rendering speed of \textsc{MiniGrid}, \pd, and \twodimlh, we could undertake a study where we trained over $6$ billion steps, to draw statistically significant inferences. 

\section{Limitations and Societal Impact}\label{sec:limitations-societal-impact}
While we have attempted to robustly evaluate our proposed ADVISOR methodology, we have primarily focused our experiments on navigational tasks where shortest path experts can be quickly computed. Further work is needed to validate that ADVISOR can be successful in other domains, \eg, imitation in interactive robotic tasks or natural language applications.

While the potential for direct negative societal impact of this work is small, it is worth noting that, in enabling agents to learn more effectively from expert supervision, this work makes imitation learning a more attractive option to RL researchers. If expert supervision is obtained from humans, RL agents trained with such data will inevitably reproduce any (potentially harmful) biases of these humans.

\section*{Acknowledgements}
This material is based upon work supported in part by the National Science Foundation under Grants No.\ 1563727, 1718221, 1637479, 165205, 1703166, 2008387, 2045586, 2106825, MRI \#1725729, NIFA award 2020-67021-32799, Samsung, 3M, Sloan Fellowship, NVIDIA Artificial Intelligence Lab, Allen Institute for AI, Amazon, AWS Research Awards, and Siebel Scholars Award. We thank Nan Jiang and Tanmay Gangwani for feedback on this work.

\bibliography{collab}
\bibliographystyle{abbrvnat}

\clearpage
\appendix

\section*{Appendix: Bridging the Imitation Gap by Adaptive Insubordination}

\setcounter{table}{0}
\renewcommand{\thetable}{\Alph{section}.\arabic{table}}
\setcounter{figure}{0}
\renewcommand{\thefigure}{\Alph{section}.\arabic{figure}}
\setcounter{algocf}{0}
\renewcommand{\thealgocf}{\Alph{section}.\arabic{algocf}}

The appendix includes theoretical extensions of ideas presented in the main paper and details of empirical analysis. We structure the appendix into the following subsections:
\begin{enumerate}
    \item[\ref{sec:formal-lighthouse-example}] A formal treatment of Ex.~\ref{ex:lighthouse} on 1D-Lighthouse.
    \item[\ref{sec:policy-averaging-proof}] Proof of Proposition~\ref{prop:policy-averaging}.
    \item[\ref{sec:other-distances}] Distance measures beyond $d^{0}_{\pi}(\pi_f)(s) = d(\pi(s), \pi_f(s))$ utilized in ADVISOR.\footnote{We overload main paper's notation $d^{0}(\pi,\pi_f)(s)$ with $d^{0}_{\pi}(\pi_f)(s)$}
    \item[\ref{sec:future-work-improving-plugins}] Future strategies for improving statistical efficiency of $d^{0}_{\piexp}(\pi^{\text{IL}}_{f})(s)$ estimator and a prospective approach towards it.
    \item[\ref{sec:additional-task-details}] Descriptions of all the tasks that we evaluate baselines on, including values for grid size, obstacles, corruption distance~\etc. We also include details about observation space for each of these tasks.
    \item[\ref{sec:lighthouse-efficiency-study}] Initial results showing that ADVISOR can outperform behavior cloning even when there is no imitation gap.
    \item[\ref{sec:additional-baseline-details}] Additional details about nature of learning, expert supervision and hyperparameters searched for each baseline introduced in~\secref{sec:baselines}.
    \item[\ref{sec:arch-details}] Details about the underlying model architecture for all baselines across different tasks.
    \item[\ref{sec:hyperparameter-tuning}] Methodologies adopted for ensuring fair hyperparameter tuning of previous baselines when comparing ADVISOR to them.
    \item[\ref{sec:training-imp-details}] Training implementation including maximum steps per episode, reward structure and computing infrastructure adopted for this work. We clearly summarize all structural and training hyperparameters for better reproducibility.
    
    \item[\ref{sec:more-plots}] Additional results including plots for all tasks to supplement~\figref{fig:minigrid-tasks}, 
    a table giving an expanded version of the \tabref{tab:pd-and-mingrid-results}, and learning curves to supplement \tabref{tab:complex-tasks}.
    
\end{enumerate}

\section{Additional Information}

\subsection{Formal treatment of Example~\ref{ex:lighthouse}}
\label{sec:formal-lighthouse-example}
Let $N\geq 1$ and consider a 1-dimensional grid-world with states $\cS = \{-N, N\} \times \{0, \ldots, T\} \times \{-N,\ldots, N\}^T$. Here $g\in\{-N, N\}$ are possible goal positions, elements $t\in\{0, \ldots, T\}$ correspond to the episode's current timestep, and $(p_i)_{i=1}^T\in\{-N,\ldots, N\}^T$ correspond to possible agent trajectories of length $T$. Taking action $a\in\cA=\{\text{left}, \text{right}\}=\{-1,1\}$ in state $(g, t, (p_i)_{i=1}^T)\in\cS$ results in the deterministic transition to state
$(g, t+1, (p_1,\dots,p_t, \text{clip}(p_t+a,-N, N), 0, \ldots, 0))$. 
An episode start state is chosen uniformly at random from the set $\{(\pm N, 0, (0, \ldots, 0))\}$ and the goal of the agent is to reach some state $(g, t, (p_i)_{i=1}^T)$ with $p_t=g$ in the fewest steps possible. We now consider a collection of filtration functions $f^i$, that allow  the agent to see spaces up to $i$ steps left/right of its current position but otherwise has perfect memory of its actions. See Figs.~\ref{fig:lighthouse-f1-obs},~\ref{fig:lighthouse-f2-obs} for examples of $f^1$- and $f^2$-restricted observations. For $0\leq i\leq N$ we define $f^i$ so that
\begin{align}
    f^i(g, t, (p_i)_{i=1}^T) &= ((\ell_0, \dots, \ell_t), (p_{1} - p_{0}, \ldots, p_{t} - p_{t-1})) \quad \text{and}  \\
    \ell_j &= (1_{[p_j + k=N]} - 1_{[p_j + k=-N]} \mid k\in\{-i, \ldots, i\}) \quad \text{for $0\leq j\leq t$.}
\end{align}
Here $\ell_j$ is a tuple of length $2\cdot i + 1$ and corresponds to the agent's view at timestep $j$ while $p_{k+1} - p_{k}$ uniquely identifies the action taken by the agent at timestep $k$. Let $\piexp$ be the optimal policy given full state information so that $\piexp(g, t, (p_i)_{i=1}^T) = (1_{[g=-N]}, 1_{[g=N]})$ and let $\mu$ be a uniform distribution over states in $\cS$. It is straightforward to show that an agent following policy $\pi_{f^i}^{\text{IL}}$ will take random actions until it is within a distance of $i$ from one of the corners $\{-N,N\}$ after which it will head directly to the goal, see the policies highlighted in Figs. \ref{fig:lighthouse-f1-obs}, \ref{fig:lighthouse-f2-obs}. The intuition for this result is straightforward: until the agent observes one of the corners it cannot know if the goal is to the right or left and, conditional on its observations, each of these events is equally likely under $\mu$. Hence in half of these events the expert will instruct the agent to go right and in the other half to go left. The cross entropy loss will thus force $\pi_{f^i}^{\text{IL}}$ to be uniform in all such states. Formally, we will have, for $s=(g, t, (p_i)_{i=1}^T)$, $\pi^{\text{IL}}_{f^i}(s)=\piexp(s)$ if and only if $\min_{0\leq q\leq t}(p_q) - i \leq -N$ or $\max_{0\leq q\leq t}(p_q) + i \geq N$ and, for all other $s$, we have $\pi_{f^i}^{\text{IL}}(s) = (1/2, 1/2)$. In Sec.~\ref{sec:experiments}, see also Fig.~\ref{fig:lh-plots}, we train $f^i$-restricted policies with $f^j$-optimal teachers for a 2D variant of this example.\hfill$\blacksquare$

\subsection{Proof of Proposition \ref{prop:policy-averaging}}\label{sec:policy-averaging-proof}
We wish to show that the minimizer of $\mathbb{E}_\mu[-\piexp_{f^e}(S)\odot \log\pi_f(S)]$ among all $f$-restricted policies $\pi_f$ is the policy $\ol{\pi} = \mathbb{E}_{\mu}[\piexp(S)\mid f(S)]$. This is straightforward, by the law of iterated expectations and as $\pi_f(s)=\pi_f(f(s))$ by definition. We obtain 
\begin{align}
    \mathbb{E}_\mu[-\piexp_{f^e}(S)\odot \log\pi_f(S)] &= -\mathbb{E}_\mu[E_\mu[\piexp_{f^e}(S)\odot \log\pi_f(S) \mid f(S)]] \nonumber\\
    &= -\mathbb{E}_\mu[E_\mu[\piexp_{f^e}(S)\odot \log\pi_f(f(S)) \mid f(S)]]\nonumber\\
    &= -\mathbb{E}_\mu[E_\mu[\piexp_{f^e}(S)\mid f(S)]\odot \log\pi_f(f(S))]\nonumber\\
    &= \mathbb{E}_\mu[-\ol{\pi}(f(S))\odot \log\pi_f(f(S))]\;. \label{eq:policy-averaging-simplified}
\end{align}
Now let $s\in\cS$ and let $o=f(s)$. It is well known, by Gibbs' inequality, that $-\ol{\pi}(o)\odot \log\pi_f(o)$ is minimized (in $\pi_f(o)$) by letting $\pi_f(o) = \ol{\pi}(o)$ and this minimizer is feasible as we have assumed that $\Pi_f$ contains \emph{all} $f$-restricted policies. Hence it follows immediately that \equref{eq:policy-averaging-simplified} is minimized by letting $\pi_f=\ol{\pi}$ which proves the claimed proposition.

\begin{algorithm}[t]
 \rebuttal{
  \SetAlgoLined
  \SetAlgoNoEnd
  \KwIn{Trainable policies $(\pi_{f}, \pi^{\text{aux}}_{f})$, expert policy $\piexp$, rollout length $L$, environment $\cE$.}
  \KwOut{Trained policy}
  \Begin{
    Initialize the environment $\cE$\\
    $\theta\gets$ randomly initialized parameters \\
    \While{Training completion criterion not met} {
        Take $L$ steps in the environment using $\pi_{f}(\cdot;\theta)$ and record resulting rewards and observations (restarting $\cE$ whenever the agent has reached a terminal state) \\
        Evaluate $\pi^{\text{aux}}_{f}(\cdot;\theta)$ and $\piexp$ at each of the above steps \\
        $L \gets$ the empirical version of the loss from Eq. \eqref{eq:advisor-loss} computed using the above rollout\\
        Compute $\nabla_{\theta}L$ using backpropagation \\
        Update $\theta$ using $\nabla_{\theta}L$ via gradient descent \\
    }
    \Return{$\pi_f(\cdot;\theta)$}
  }
  \caption{\textbf{On-policy ADVISOR algorithm overview.} Some details omitted for clarity.} \label{alg:onpolicy-advisor}
  }
\end{algorithm}

\subsection{Other Distance Measures} \label{sec:other-distances}

As discussed in Section~\ref{sec:adaptive-method}, there are several different choices one may make when choosing a measure of distance between the expert policy $\piexp$ and an $f$-restricted policy $\pi_f$ at a state $s\in \cS$. The measure of distance we use in our experiments, $d^{0}_{\piexp}(\pi_{f})(s) = d(\piexp(s), \pi_f(s))$, has the (potentially) undesirable property that $f(s)=f(s')$ does not imply that $d^{0}_{\piexp}(\pi_{f})(s)=d^{0}_{\piexp}(\pi_{f})(s')$. While an in-depth evaluation of the merits of different distance measures is beyond this current work, we suspect that a careful choice of such a distance measure may have a substantial impact on the speed of training. The following proposition lists a collection of possible distance measures with a conceptual illustration given in \figref{fig:distance-demo}.

\begin{prop}
Let $s\in \cS$ and $o=f(s)$ and for any $0<\beta<\infty$ define, for any policy $\pi$ and $f$-restricted policy $\pi_f$,
\begin{align}
    d^{\beta}_{\mu,\pi}(\pi_f)(s)
    &= E_\mu[\left(d^{0}_{\pi}(\pi_f)(S)\right)^\beta \mid f(S) = f(s)]^{1/\beta},
\end{align}
with $d^{\infty}_{\mu,\pi}(\pi_f)(s)$ equalling the essential supremum of $d^{0}_{\pi}(\pi_f)$ under the conditional distribution $P_\mu(\cdot \mid f(S) = f(s))$. As a special case note that
\begin{align*}
    d^{1}_{\mu,\pi}(\pi_f)(s) &= E_{\mu}[d^{0}_{\pi}(\pi_f)(S) \mid f(S) = f(s)].
\end{align*}

Then, for all $\beta \geq 0$ and $s\in\cS$ (almost surely $\mu$), we have that $\pi(s) \not= \pi_f(f(s))$ if and only if $d^{\beta}_{\pi}(\pi_f)(s) > 0$.
\end{prop}
\begin{proof}
This statement follows trivially from the definition of $\pi^{\text{IL}}$ and the fact that $d(\pi, \pi') \geq 0$ with $d(\pi, \pi')=0$ if and only if $\pi=\pi'$.
\end{proof}

The above proposition shows that any $d^{\beta}$ can be used to consistently detect differences between $\piexp$ and $\pi^{\text{IL}}_f$, \ie, it can be used to detect the imitation gap. Notice also that for any $\beta>0$ we have that $d^{\beta}_{\mu,\piexp}(\pi^{\text{IL}}_{f})(s)=d^{\beta}_{\mu,\piexp}(\pi^{\text{IL}}_{f})(s')$ whenever $f(s)=f(s')$.

As an alternative to using $d^0$, we now describe how $d^{1}_{\mu,\piexp}(\pi^{\text{IL}}_{f})(s)$ can be estimated in practice during training. Let $\pi^{\text{aux}}_{f}$ be an estimator of $\pi^{\text{IL}}_{f}$ as usual. To estimate $d^{1}_{\mu,\piexp}(\pi^{\text{IL}}_{f})(s)$ we assume we have access to a function approximator $g_{\psi}:\cO_f\to \bR$ parameterized by $\psi\in\Psi$, \eg, a neural network. Then we estimate $d^{1}_{\mu,\piexp}(\pi^{\text{IL}}_{f})(s)$ with $g_{\hat{\psi}}$ where \smash{$\hat{\psi}$} is taken to be the minimizer of the loss
\begin{align}
    \cL_{\mu,\piexp,\pi^{\text{aux}}_{f}}(\psi) = E_{\mu}\Big[\Big(d(\piexp(S), \pi^{\text{aux}}_{f}(f(S))) - g_{\psi}(f(S))\Big)^2\Big].
\end{align}
The following proposition then shows that, assuming that $d^{1}_{\mu,\piexp}(\pi^{\text{aux}}_{f})\in \{g_{\psi}\mid \psi\in\Psi\}$, $g_{\hat{\psi}}$ will equal $d^{1}_{\mu,\piexp}(\pi^{\text{aux}}_{f})$ and thus $g_{\hat{\psi}}$ may be interpreted as a plug-in estimator of $d^{1}_{\mu,\piexp}(\pi^{\text{IL}}_{f})$.

\begin{prop}\label{prop:optimizing-the-right-loss}
For any $\psi\in\Psi$, 
\begin{align*}
    \cL_{\mu,\piexp,\pi^{\text{aux}}_{f}}(\psi) &= E_{\mu}[(d^{1}_{\mu,\piexp}(\pi^{\text{aux}}_{f})(S) - g_{\psi}(f(S)))^2] + c,
\end{align*}
where $c=E_{\mu}[(d(\piexp(S), \pi^{\text{aux}}(f(S))) - d^{1}_{\mu,\piexp,\hat{\pi}}(S))^2]$ is constant in $\psi$ and this implies that if $d^{1}_{\mu,\piexp}(\pi^{\text{aux}}_{f})\in \{g_{\psi}\mid\psi \in\Psi\}$ then $g_{\hat{\psi}} = d^{1}_{\mu,\piexp}(\pi^{\text{aux}}_{f})$.
\end{prop}

\begin{proof}
In the following we let $O_f = f(S)$. We now have that
\begin{align*}
    E_{\mu}&[\big(d(\piexp(S), \pi^{\text{aux}}_{f}(O_f)) - g_{\psi}(O_f)\big)^2] \\
    &= E_{\mu}[\big( (d(\piexp(S), \pi^{\text{aux}}_{f}(O_f)) - d^{1}_{\mu,\piexp}(\pi^{\text{aux}}_{f})(S)) + (d^{1}_{\mu,\piexp}(\pi^{\text{aux}}_{f})(S) -g_{\psi}(O_f)) \big)^2] \\
    &= E_{\mu}[(d(\piexp(S), \pi^{\text{aux}}_{f}(O_f)) - d^{1}_{\mu,\piexp}(\pi^{\text{aux}}_{f})(S))^2] + E_\mu[(d^{1}_{\mu,\piexp}(\pi^{\text{aux}}_{f})(S) -g_{\psi}(O_f)))^2] \\
    & \quad + 2\cdot E_\mu[((d(\piexp(S), \pi^{\text{aux}}_{f}(O_f)) - d^{1}_{\mu,\piexp}(\pi^{\text{aux}}_{f})(S))\cdot (d^{1}_{\mu,\piexp}(\pi^{\text{aux}}_{f})(S) -g_{\psi}(O_f)))] \\
    &= c + E_\mu[(d^{1}_{\mu,\piexp}(\pi^{\text{aux}}_{f})(S) -g_{\psi}(O_f)))^2] \\
    & \quad + 2\cdot E_\mu[((d(\piexp(S), \pi^{\text{aux}}_{f}(O_f)) - d^{1}_{\mu,\piexp}(\pi^{\text{aux}}_{f})(S))\cdot (d^{1}_{\mu,\piexp}(\pi^{\text{aux}}_{f})(S) -g_{\psi}(O_f)))].
\end{align*}
Now as as $d^{1}_{\mu,\piexp}(\pi^{\text{aux}}_{f})(s) = d^{1}_{\mu,\piexp}(\pi^{\text{aux}}_{f})(s')$ for any $s,s'$ with $f(s)=f(s')$ we have that $d^{1}_{\mu,\piexp}(\pi^{\text{aux}}_{f})(S) -g_{\psi}(O_f)$ is constant conditional on $O_f$ and thus
\begin{align*}
    E_\mu&[(d(\piexp(S), \pi^{\text{aux}}_{f}(O_f)) - d^{1}_{\mu,\piexp}(\pi^{\text{aux}}_{f})(S))\cdot (d^{1}_{\mu,\piexp}(\pi^{\text{aux}}_{f})(S) -g_{\psi}(O_f)) \mid O_f] \\
    &= E_\mu[(d(\piexp(S), \pi^{\text{aux}}_{f}(O_f)) - d^{1}_{\mu,\piexp}(\pi^{\text{aux}}_{f})(S)\mid O_f] \cdot E_\mu[d^{1}_{\mu,\piexp}(\pi^{\text{aux}}_{f})(S) -g_{\psi}(O_f) \mid O_f] \\
    &= E_\mu[d^{1}_{\mu,\piexp}(\pi^{\text{aux}}_{f})(S) - d^{1}_{\mu,\piexp}(\pi^{\text{aux}}_{f})(S)\mid O_f] \cdot E_\mu[d^{1}_{\mu,\piexp}(\pi^{\text{aux}}_{f})(S) -g_{\psi}(O_f) \mid O_f] \\
    &= 0.
\end{align*}
Combining the above results and using the law of iterated expectations  gives the desired result.
\end{proof}

\begin{figure}
    \centering
    {
      \phantomsubcaption\label{fig:distance-demo-1}
      \phantomsubcaption\label{fig:distance-demo-2}
      \phantomsubcaption\label{fig:distance-demo-3}
      \phantomsubcaption\label{fig:distance-demo-4}
    }
    \includegraphics[width=\textwidth]{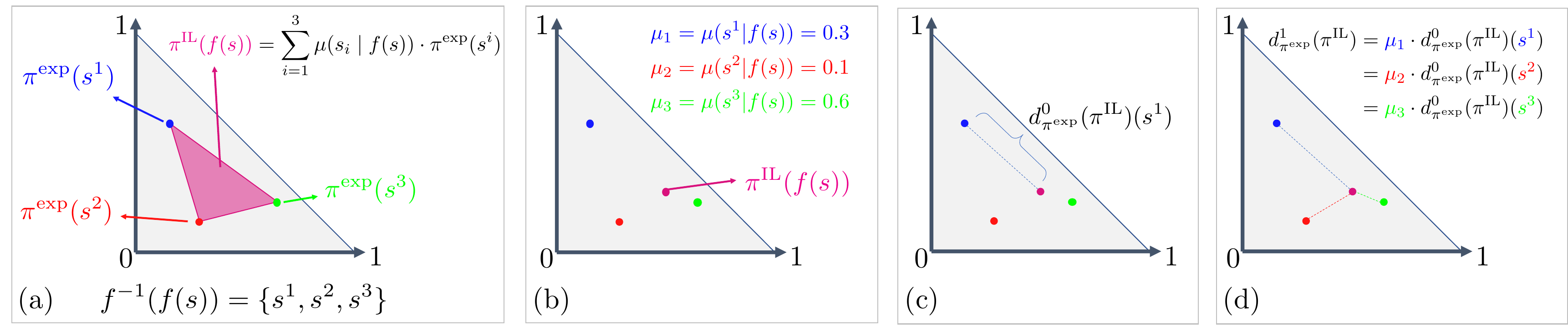}

\caption{\textbf{Concept Illustration.} Here we illustrate several of the concepts from our paper. Suppose our action space $\cA$ contains three elements. Then for any $s\in\cS$ and policy $\pi$, the value $\pi(s)$ can be represented as a single point in the 2-dimensional probability simplex $\{(x,y)\in\bR^2\mid x\geq 0, y\geq 0, x+y \leq 1\}$ shown as the grey area in (a). Suppose that the fiber $f^{-1}(f)$ contains the three unique states $s^1,s^2,$ and $s^3$. In (a) we show the hypothetical values of $\pi^{\text{exp}}$ when evaluated at these points. Proposition~\ref{prop:policy-averaging} says that $\pi^{\text{IL}}(s)$ lies in the convex hull of $\{\piexp(s^i)\}_{i=1}^3$ visualized as a magenta triangle in (a). Exactly where $\pi^{\text{IL}}(s)$ lies depends on the probability measure $\mu$, in (b) we show how a particular instantiation of $\mu$ may result in a realization of $\pi^{\text{IL}}(s)$ (not to scale). (c) shows how $d^1_{\piexp}$ measures the distance between $\piexp(s^1)$ and $\pi^{\text{IL}}(s^1)$. Notice that it ignores $s^2$ and $s^3$. In (d), we illustrate how $d^0_{\piexp}$ produces a ``smoothed'' measure of distance incorporating information about all $s^i$.
}
    \label{fig:distance-demo}
\end{figure}

\subsection{Future Directions in Improving Distance Estimators} \label{sec:future-work-improving-plugins}

In this section we highlight possible directions towards improving the estimation of $d^{0}_{\piexp}(\pi^{\text{IL}}_{f})(s)$ for $s\in\cS$. As a comprehensive study of these directions is beyond the scope of this work, our aim in this section is intuition over formality. We will focus on $d^0$ here but similar ideas can be extended to other distance measures, \eg, those in Sec.~\ref{sec:other-distances}.

As discussed in the main paper, we estimate $d^{0}_{\piexp}(\pi^{\text{IL}}_{f})(s)$ by first estimating $\pi^{\text{IL}}_{f}$ with $\pi^{\text{aux}}_{f}$ and then forming the ``plug-in'' estimator $d^{0}_{\piexp}(\pi^{\text{aux}}_{f})(s)$. For brevity, we will write $d^{0}_{\piexp}(\pi^{\text{aux}}_{f})(s)$ as $\smhat{d}$. While such plug-in estimators are easy to estimate and conceptually compelling, they need not be statistically efficient. Intuitively, the reason for this behavior is because we are spending too much effort in trying to create a high quality estimate $\pi^{\text{aux}}_{f}$ of $\pi^{\text{IL}}_{f}$ when we should be willing to sacrifice some of this quality in service of obtaining a better estimate of $d^{0}_{\piexp}(\pi^{\text{IL}}_{f})(s)$. Very general work in this area has brought about the targeted maximum-likelihood estimation (TMLE) \citep{van2016one} framework. Similar ideas may be fruitful in improving our estimator $\smhat{d}$.

Another weakness of $\smhat{d}$ discussed in the main paper is that is not prospective. In  the main paper we assume, for readability,  that we have trained the estimator $\pi^{\text{aux}}_{f}$ before we train our main policy. In practice, we train $\pi^{\text{aux}}_{f}$ alongside our main policy. Thus the quality of $\pi^{\text{aux}}_{f}$ will improve throughout training. To clarify, suppose that, for $t\in[0,1]$, $\pi^{\text{aux}}_{f, t}$ is our estimate of $\pi^{\text{IL}}_{f}$ after $(100\cdot t) \%$ of training has completed. Now suppose that $(100\cdot t) \%$ of training has completed and we wish to update our main policy using the ADVISOR loss given in  \equref{eq:advisor-loss}. In our current approach we estimate $d^{0}_{\piexp}(\pi^{\text{IL}}_{f})(s)$ using $d^{0}_{\piexp}(\pi^{\text{aux}}_{f,t})(s)$ when, ideally, we would prefer to use $d^{0}_{\piexp}(\pi^{\text{aux}}_{f,1})(s)$ from the end of training. Of course we will not know the value of $d^{0}_{\piexp}(\pi^{\text{aux}}_{f,1})(s)$ until the end of training but we can, in principle, use time-series methods to estimate it. To this end, let $q_\omega$ be a time-series model with parameters $\omega\in\Omega$ (\eg, $q_\omega$ might be a recurrent neural network) and suppose that we have stored the model checkpoints $(\pi^{\text{aux}}_{f,i/K} \mid i/K \leq t)$. We can then train $q_\omega$ to perform forward prediction, for instance to minimize
\begin{align*}
    \sum_{j=1}^{\lfloor t \cdot K\rfloor} \Big(d^{0}_{\piexp}(\pi^{\text{aux}}_{f,j/K})(s) - q_\omega(s, (\pi^{\text{aux}}_{f,i/K}(s))_{i=1}^{j-1})\Big)^2\;,
\end{align*}
and then use this trained $q_\omega$ to predict the value of $d^{0}_{\piexp}(\pi^{\text{aux}}_{f,1})(s)$. The advantage of this prospective estimator $q_\omega$ is that it can detect that the auxiliary policy will eventually succeed in exactly imitating the expert in a given state and thus allow for supervising the main policy with the expert cross entropy loss earlier in training. The downside of such a method: it is significantly more complicated to implement and requires running inference using saved model checkpoints.

\subsection{Additional Task Details}\label{sec:additional-task-details}

In \Cref{fig:tasks} we gave a quick qualitative glimpse at the various tasks we use in our experiments. Here, we provide additional details for each of them along with information about observation space associated with each task. For training details for the tasks, please see \secref{sec:training-imp-details}. Our experiments were primarily run using the AllenAct learning framework~\cite{AllenAct}, see \url{AllenAct.org} for details and tutorials.

\subsubsection{PoisonedDoors (\pd)}\label{sec:poisoned-doors-task-description-full} This environment is a reproduction of our example from~\secref{sec:introduction}. An agent is presented with $N=4$ doors $d_1,\dots,d_4$. Door $d_1$ is locked, requiring a fixed $\{0,1,2\}^{10}$ code to open, but always results in a reward of 1 when opened. For some randomly chosen $j\in\{2,3,4\}$, opening door $d_j$ results in a reward of $2$ and for $i\not\in\{1,j\}$, opening door $d_i$ results in a reward of $-2$. The agent must first choose a door after which, if it has chosen door 1, it must enter the combination (receiving a reward of 0 if it enters the incorrect combination) and, otherwise, the agent immediately receives its reward. See Fig.~\ref{fig:babyai-tasks-poisoned-doors}. 

\subsubsection{2D-Lighthouse (\twodimlh)}

\begin{wrapfigure}[16]{r}{0.35\linewidth}
    \centering
    \includegraphics[width=0.9\linewidth]{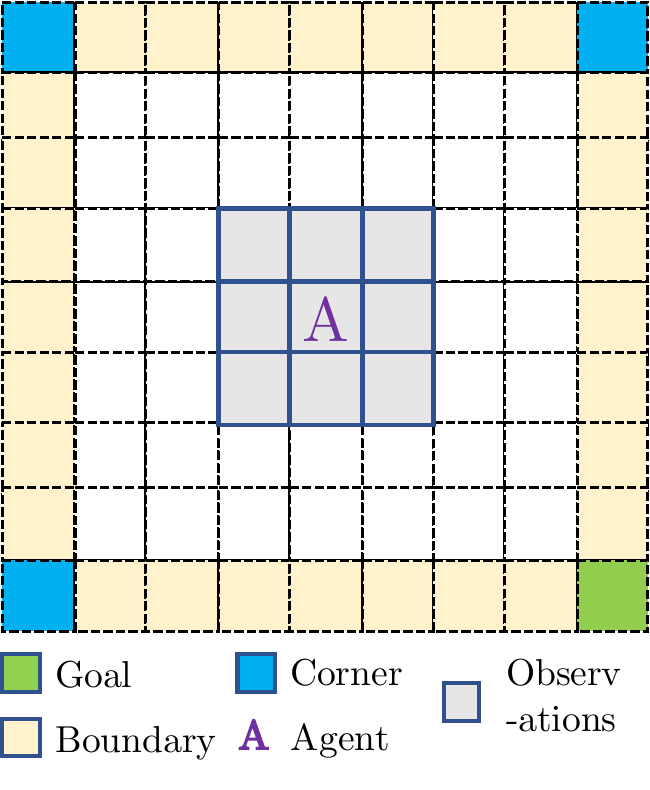}
    \caption{\textsc{2D-Lighthouse}}
    \label{fig:babyai-tasks-lighthouse}
\end{wrapfigure}

2D variant of the exemplar grid-world task introduced in Ex.~\ref{ex:lighthouse}, aimed to empirically verify our analysis of the imitation gap. A reward awaits at a randomly chosen corner of a square grid of size $2N+1$ and the agent can only see the local region, a square of size $2i+1$ about itself (an $f^i$-restricted observation). Additionally, all $f^i$ allow the agent access to it's previous action. As explained in Ex.~\ref{ex:lighthouse}, we experiment with optimizing $f^{i}$-policies when given supervision from $f^j$-optimal teachers (\ie, experts that are optimal when restricted to $f^j$-restricted observations). See Fig.~\ref{fig:babyai-tasks-lighthouse} for an illustration.

\subsubsection{LavaCrossing (LC)} Initialized on the top-left corner the agent must navigate to the bottom-right goal location. There exists at least one path from start to end, navigating through obstacles. Refer to~\figref{fig:tasks} where, for illustration, we show a simpler grid. Here the episode terminates if the agent steps on any of the \texttt{lava} obstacles. This \textsc{LC} environment has size $25\times 25$ with $10$ \texttt{lava} rivers (`S25, N10' as per the notation of~\citep{gym_minigrid}), which are placed vertically or horizontally across the grid. The expert is a shortest path agent with access to the entire environment's connectivity graph and is implemented via the \texttt{networkx} python library.

\subsubsection{WallCrossing (WC)} Similar to \textsc{LavaCrossing} in structure and expert, except that obstacles are \texttt{walls} instead of \texttt{lava}. Unlike \texttt{lava} (which immediately kills the agent upon touching), the agent may run into \texttt{walls} without consequence (other than wasting time). Our environment is of size $25\times 25$ with $10$ \texttt{walls} (`S25, N10').

\subsubsection{WC/LC Switch}
In this task the agent faces a more challenging filtration function. In addition to navigational actions, agents for this task have a `switch' action. Using this switch action, the agents can switch-on the lights of an otherwise darkened environment which is implemented as an observation tensor of all zeros. In \textsc{WC}, even in the dark, an agent can reach the target by taking random actions with non-negligible probability. Achieving this in \textsc{LC} is nearly impossible as random actions will, with high probability, result in stepping into \texttt{lava} and thereby immediately end the episode.

We experiment with two variants of this `switch' -- \textsc{Once} and \textsc{Faulty}. In the \textsc{Once Switch} variant, once the the `switch' action is taken, the lights remain on for the remainder of the episode. This is implemented as the unaffected observation tensor being available to the agent. In contrast, in the \textsc{Faulty Switch} variant, taking the `switch' action will only turn the lights on for a single timestep. This is implemented as observations being available for one timestep followed by zero tensors (unless the `switch' action is executed again). 

The expert for these tasks is the same as for \textsc{WC} and \textsc{LC}. Namely, the expert always takes actions along the shortest path from the agents current position to the goal and is unaffected by whether the light is on or off. For the expert-policy-based methods this translates to the learner agent getting perfect (navigational) supervision while struggling in the dark, with no cue for trying the switch action. For the expert-demonstrations-based methods this translates to the demonstrations being populated with blacked-out observations paired with perfect actions: such actions are, of course, difficult to imitate. As \textsc{Faulty} is more difficult than \textsc{Once} (and \textsc{LC} more difficult than \textsc{WC}) we set grid sizes to reduce the difference in difficulty between tasks. In particular, we choose to set \textsc{WC Once Switch} on a (S25, N10) grid and the \textsc{LC Once Switch} on a (S15, N7) grid. Moreover, \textsc{WC Faulty Switch} is set with a (S15, N7) grid and \textsc{LC Faulty Switch} with a (S9, N4) grid.

\subsubsection{WC/LC Corrupt} In the \textsc{Switch} task, we study agents with observations affected by a challenging filtration function. In this task we experiment with corrupting the expert's actions. The expert policy flips over to a random policy when the expert is $N_C$ steps away from the goal. For the expert-policy-based method this translates to the expert outputting uniformly random actions once it is within $N_C$ steps from the target. For the expert-demonstrations-based methods this translates to the demonstrations consisting of some valid (observation, expert action) tuples, while the tuples close to the target have the expert action sampled from a uniform distribution over the action space. \textsc{WC Corrupt} is a (S25, N10) grid with $N_C = 15$, while the \textsc{LC Corrupt} is significantly harder, hence is a (S15, N7) grid with $N_C = 10$.

\subsubsection{PointGoal Navigation}\label{sec:pointnav-task} In PointGoal Navigation, a randomly spawned agent must navigate to a goal specified by a relative-displacement vector. The observation space is composed of rich egocentric RGB observations ($256{\times}256{\times}3$) with a limited field of view. The action space is 
$\{$\texttt{move\_ahead}, \texttt{rotate\_right}, \texttt{rotate\_left}, \texttt{stop}$\}$.
The task was formulated by \cite{anderson2018evaluation} and implemented for the \textsc{AIHabitat} simulator by \cite{habitat19iccv}. Our reward structure, train/val/test splits, PointNav dataset, and implementation follow \cite{habitat19iccv}. 
RL agents are trained using PPO following authors' implementation\footnote{https://github.com/facebookresearch/habitat-lab}. The IL agent is trained with on-policy behavior cloning using the shortest-path action. A static combination of the PPO and BC losses (\ie a simple sum of the PPO loss and IL cross entropy loss) is also used a competing baseline for ADVISOR. Note that the agent observes a filtered egocentric observation while the shortest-path action is inferred from the entire environment state leading to a significant imitation gap. We train on the standard Gibson set of 76 scenes, and report metrics as an average over the val.\ set consisting of 14 unseen scenes in \textsc{AIHabitat}. We use a budget of 50 million frames, \ie, ${\sim}$2 days of training on  4 NVIDIA TitanX GPUs, and 28 CPUs for each method.

\subsubsection{ObjectGoal Navigation}\label{sec:objectnav-task} In ObjectGoal Navigation within the RoboTHOR environment, a randomly spawned agent must navigate to a goal specified by an object category. In particular, the agent must search it's environment to find an object of the given category and take a 
\texttt{stop} action (which ends the episode regardless of success) when that object is within 1m of the agent and visible. The observation space is composed of rich egocentric RGB observations ($300{\times}400{\times}3$) with a limited field of view. The action space is 
$\{$\texttt{move\_ahead}, \texttt{rotate\_right}, \texttt{rotate\_left}, \texttt{look\_up}, ,\texttt{look\_down}, \texttt{stop}$\}$.
The \onav task within the RoboTHOR environment was proposed by \cite{DeitkeEtAl2020}, we use the version of this task corresponding to the 2021 RoboTHOR ObjectNav Challenge\footnote{https://ai2thor.allenai.org/robothor/cvpr-2021-challenge} and use this challenge's reward structure, dataset, train/val/test splits, and their baseline model architecture. 
This challenge provides implementations of PPO and DAgger where the DAgger agent is trained with supervision coming from a shortest-path expert. We implement our ADVISOR methodology (with no teacher forcing) as well as a baseline where we simply sum PPO and IL losses. We use a budget of 100 million frames, \ie, ${\sim}$2-5 days of training, 8 NVIDIA TitanX GPUs, and 56 CPUs for each method. At every update step we use 60 rollouts of length 128 and perform 4 gradient steps with the rollout.

\subsubsection{Cooperative Navigation}
\label{sec:coopnav-task}
In Cooperative Navigation, there are three agents and three landmarks. The goal of the three agents is to cover the three landmarks. 
Agents are encouraged to move toward uncovered landmarks and get penalized when they collide with each other. Agents have limited visibility range. The agents can only observe other agents and landmarks within its visibility range (euclidean distance to the agent). 
The action space has five dimensions. The first dimension is no-op, and the other four dimensions represent the forward, backward, left, and right  force applied to the agent. The RL agents are trained with MADDPG~\cite{LoweNIPS2017} with a permutation invariant critic~\cite{LiuCORL2019}. The IL agents are trained using DAgger. The experts are pre-trained RL agents with no limits to their visibility range. Following~\cite{LoweNIPS2017, LiuCORL2019}, we use a budge of $1.5$ million environment steps. We use one NVIDIA GTX1080 and 2 CPUs to train these agents.

\subsubsection{Observation spaces} \label{sec:appendix-obs-space}

\noindent \textbf{\twodimlh.} Within our \twodimlh environment we wish to train our agent in the context of Proposition \ref{prop:policy-averaging} so that the agent may learn any $f$-restricted policy. As the \twodimlh environment is quite simple, we are able to uniquely encode the state observed by an agent using a $4^4\cdot 5^2=6400$ dimensional $\{0,1\}$-valued vector such that any $f$-restricted policy can be represented as a linear function applied to this observation (followed by a soft-max).\footnote{As the softmax function prevents us from learning a truly deterministic policy we can only learn a policy arbitrarily close to such policies. In our setting, this distinction is irrelevant.}

\noindent \textbf{\pd.} Within the \pd environment the agent's observed state is very simple: at every timestep the agent observes an element of $\{0,1,2,3\}$ with 0 denoting that no door has yet been chosen, 1 denoting that the agent has chosen door $d_1$ but has not begun entering the code, $2$ indicating that the agent has chosen door $d_1$ and has started entering the code, and $3$ representing the final terminal state after a door has been opened or combination incorrectly entered.

\noindent \textbf{\minigrid.} The \minigrid environments~\citep{gym_minigrid} enable agents with an egocentric ``visual'' observation which, in practice, is an integer tensor of shape $7\times7\times3$, where the channels contain integer labels corresponding to the cell's type, color, and state. Kindly see~\citep{gym_minigrid,chevalier2018babyai} for details. For the above tasks, the cell types belong to the set of (\texttt{empty, lava, wall, goal}).

\noindent \textbf{\textsc{PointNav}.} Agents in the \pnav task observe, at every step, egocentric RGB observations ($256{\times}256{\times}3$) of their environment along with a relative displacement vector towards the goal (\ie a 2d vector specifying the location of the goal relative the goal). See \Cref{fig:tasks} for an example of one such egocentric RGB image.

\noindent \textbf{\textsc{ObjectNav}.} Agents in the \onav task observe, at every step, egocentric RGB observations ($300{\times}400{\times}3$) of their environment along with an object category (\eg ``BaseballBat'') specifying their goal. See \Cref{fig:tasks} for an example of one such egocentric RGB image. Note that agents in the \textsc{ObjectNav} task are generally also allowed access to egocentric depth frames, we do not use these depth frames in our experiments as their use slows simulation speed.

\noindent \textsc{CoopNav}.
At each step, each agent in \textsc{CoopNav} task observes a 14-dimensional vector, which contains the absolute location and speed of itself, the relative locations to the three landmarks, and the relative location to other two agents.  

\subsection{ADVISOR can outperform BC in the no-imitation-gap setting} \label{sec:lighthouse-efficiency-study}

Recall the setting of our \twodimlh experiments in Section \ref{sec:results} where we train $f^i$-restricted policies (\ie, an agent that can see $i$ grid locations away) using $f^j$-optimal teachers. In particular, we train 25 policies on each $i,j$ pair where for $1\leq i\leq j\leq 15$ and $i,j$ are both odd. Each trained policy is then evaluated on 200 random episodes and we record average performance across various metrics across these episodes. Complementing Fig. \ref{fig:lh-plots} from the main paper, Fig. \ref{fig:app-lh-plots} shows the box plots of the trained policies average episode lengths, lower being better, when training with BC, BC$\to$ PPO, ADVISOR, and PPO (PPO does not use expert supervision so we simply report the performance of PPO trained $f^i$-restricted policies for each $i$).

As might be expected: ADVISOR has consistently low episode lengths across all $i,j$ pairs suggesting that ADVISOR is able to mitigate the impact of the imitation gap. One question that is not well-answered by Fig. \ref{fig:app-lh-plots} is that of the relative performance of ADVISOR and IL methods when \emph{there is no imitation gap}, namely the $i=j$ case. As ADVISOR requires the training of an auxiliary policy in addition (but, in parallel) to a main policy, we test the sample efficiency of ADVISOR head-on with IL methods.
Table \ref{tab:lighthouse-efficiency} records the percentage of runs in which ADV, BC, and $\dagger$ attain near optimal (within 5\%) performance when trained in the no-imitation-gap setting (\ie $i=j$) for different grid visibility $i$. 
We find that only ADVISOR consistently reaches near-optimal performance within the budget of 300{,}000 training steps. We suspect that this is due to the RL loss encouraging early exploration that results in the agent more frequently entering states where imitation learning is easier. This interpretation is supported by the observation that ADV, BC, and $\dagger$ all consistently reach near-optimal performance when $i$ is very small (almost all states look identical so exploration can be of little help) and when $i$ is quite large (the agent can see nearly the whole environment so there is no need to explore). While we do no expect this trend to hold in all cases, indeed there are likely many cases where pure-IL is more effective than ADV in the no-imitation-gap setting, it is encouraging to see that ADV can bring benefits even when there is no imitation gap.

\begin{table}[t]
\footnotesize
\centering
\begin{tabular}{ccccccccc}
\toprule
\textbf{Method}           & \multicolumn{8}{c}{\textbf{\% converged to near optimal performance}} \\
           & $i=1$    & 3    & 5    & 7    & 9    & 11   & 13   & 15 \\\midrule
ADV & 1    & 1    & 1    & 1    & 1    & 1    & 1    & 1  \\
BC           & 1    & 0.72 & 0.52 & 0.72 & 0.68 & 0.84 & 0.96 & 1  \\
$\dagger$       & 0.88 & 0.56 & 0.24 & 0.08 & 0.52 & 0.96 & 1    & 1 \\\bottomrule
\end{tabular}
\vspace{2mm}
\caption{\textbf{Comparing efficiency of IL \vs ADVISOR in \twodimlh}. Here we report the percentage of runs (of 25 runs per (method, $i$) pair) that various methods converged to near-optimal performance (within 5\% of optimal) with a budget of 300{,}000 training steps. Here $i$ corresponds to an $f^i$-restricted (student) policy trained with expert supervision from an $f^i$-optimal teacher (\ie the `no-imitation-gap' setting).}\label{tab:lighthouse-efficiency}
\end{table}

\subsection{Additional baseline details} \label{sec:additional-baseline-details}

\subsubsection{Baselines details for \twodimlh, \pd, and \minigrid tasks}

In~\tabref{tab:baseline_details}, we include details about the baselines considered in this work, including -- purely RL ($1$), purely IL ($2-4,9$), a sequential combination of IL/RL ($6-8$), static combinations of IL/RL ($5,10$), a method that uses expert demonstrations to generate rewards for reward-based RL (\ie GAIL, $11$), and our dynamic combinations ($12-15$). Our imitation learning baselines include those which learn from both expert policy (\ie an expert action is assumed available for any state) and expert demonstrations (offline dataset of pre-collected trajectories).

In our study of hyperparameter robustness (using the \pd and \minigrid tasks) the hyperparameters (hps) we consider for optimization have been chosen as those which, in preliminary experiments, had a substantial impact on model performance. This includes the learning rate (lr), portion of the training steps devoted to the first stage in methods with two stages (stage-split), and the temperature parameter in the ADVISOR weight function ($\alpha$).\footnote{See~\secref{sec:adaptive-method} for definition of the weight function for ADVISOR.} Note that, the random environment seed also acts as an implicit hyperparameter. We sample hyperparameters uniformly at random from various sets. In particular, we sample lr from $[10^{-4}, 0.5)$ on a log-scale, stage-split from $[0.1, 0.9)$, and $\alpha$ from $\{4, 8, 16, 32\}$.

\begin{savenotes}
\begin{table}[]
    \centering
    \begin{tabular}{lllll}
         \toprule
         \textbf{\#} & \textbf{Method} & \textbf{IL/RL} & \textbf{Expert supervision} & \textbf{Hps. searched}\\
         \midrule
         1 & $\text{PPO}$ & RL & Policy & lr\\
         2 & BC & IL & Policy & lr\\
         3 & $\dagger$ & IL & Policy & lr, stage-split\\
         4 & $\text{BC}^{\text{tf=1}}$ & IL & Policy\footnote{While implemented with supervision from expert policy, due to the teacher forcing being set to $1.0$, this method can never explore beyond states (and supervision) in expert demonstrations.} & lr\\
         5 & $\text{BC}+\text{PPO}$ & IL\&RL & Policy & lr\\
         6 & $\text{BC}\rightarrow\text{PPO}$ & IL$\rightarrow$RL & Policy & lr, stage-split\\
         7 & $\dagger\rightarrow\text{PPO}$ & IL$\rightarrow$RL & Policy & lr, stage-split\\
         8 & $\text{BC}^{\text{tf=1}}\rightarrow\text{PPO}$ & IL$\rightarrow$RL & Policy & lr, stage-split\\
         9 & $\text{BC}^{\text{demo}}$ & IL & Demonstrations & lr\\
         10 & $\text{BC}^{\text{demo}}+\text{PPO}$ & IL\&RL & Demonstrations & lr\\
         11 & GAIL & IL\&RL & Demonstrations & lr\\
         \midrule
         12 & $\text{ADV}$ & IL\&RL & Policy & lr, $\alpha$\\
         13 & $\dagger\rightarrow\text{ADV}$ & IL\&RL & Policy & lr, $\alpha$, stage-split\\
         14 & $\text{BC}^{\text{tf=1}}\rightarrow\text{ADV}$ & IL\&RL & Policy & lr, $\alpha$, stage-split\\
         15 & $\text{BC}^{\text{demo}}+\text{ADV}$ & IL\&RL & Demonstrations & lr, $\alpha$\\
         \bottomrule
    \end{tabular} \vspace{1mm}
    \caption{\textbf{Baseline details.} IL/RL: Nature of learning, Expert supervision: the type of expert supervision leveraged by each method, Hps.\ searched: hps.\ that were randomly searched over, fairly done with the same budget (see~\secref{sec:hyperparameter-tuning} for details).}
    \label{tab:baseline_details}
\end{table}
\end{savenotes}

In the below we give additional detailis regarding the GAIL and $\text{ADV}^{\text{demo}}+\text{PPO}$ methods.

\noindent\textbf{Generative adversarial imitation learning (GAIL).} For a comprehensive overview of GAIL, please see \cite{ho2016generative}. Our implementation closely follows that of Ilya Kostrikov \cite{KostrikovPytorchRL}. We found GAIL to be quite unstable without adopting several critical implementation details. In particular, we found it critical to (1) normalize rewards using a (momentum-based) running average of the standard deviation of past returns and (2) provide an extensive ``warmup'' period in which the discriminator network is pretrained. Because of the necessity of this ``warmup period'', our GAIL baseline observes more expert supervision and is given a budget of substantially more gradient steps than all other methods. Because of this, our comparison against GAIL \emph{disadvantages our ADVISOR method}. Despite this disadvantage, ADVISOR still outperforms.

\noindent\textbf{The $\text{ADV}^{\text{demo}}+\text{PPO}$ method.} As described in the main paper, the $\text{ADV}^{\text{demo}}+\text{PPO}$ method attempts to bring the benefits of our ADVISOR methodology to the setting where expert demonstrations are available but an expert policy (\ie, an expert that can be evaluated at arbitrary states) is not. Attempting to compute the ADVISOR loss (recall Eq.~\eqref{eq:advisor-loss}) on off-policy demonstrations is complicated however, as our RL loss assumes  access to on-policy demonstrations. In theory, importance sampling methods, see, \eg, \citep{MahmoodHasseltSutton2014}, can be used to ``reinterpret'' expert demonstrations as though they were on-policy. But such methods are known to be somewhat unstable, non-trivial to implement, and may require information about the expert policy that we do not have access to. For these reasons, we choose to use a simple solution: when computing the ADVISOR loss on expert demonstrations we ignore the RL loss. Thus $\text{ADV}^{\text{demo}}+\text{PPO}$ works by looping between two phases:
\begin{itemize}
    \item Collect an (on-policy) rollout using the agent's policy, compute the PPO loss for this rollout and perform gradient descent on this loss to update the parameters.
    \item Sample a rollout from the expert demonstrations and, using this rollout, compute the demonstration-based ADVISOR loss
    \begin{align}
        \cL^{\text{ADV-demo}}(\theta) = \mathbb{E}_{\text{demos.}}[w(S) \cdot CE(\piexp(S), \pi_f(S;\theta))],
    \end{align}
    and perform gradient descent on this loss to update the parameters.
\end{itemize}

\subsubsection{Baselines used in \pnav experiments}

Our \pnav baselines are described in \Cref{sec:coopnav-task}. See also Table \ref{tab:hyperparams-onav}.

\subsubsection{Baselines details for \onav experiments}

Our \onav baselines are described in \Cref{sec:objectnav-task}. See also Table \ref{tab:hyperparams-onav}.

\subsubsection{Baselines used in \cnav experiments}
Our \cnav baselines are described in \Cref{sec:coopnav-task}. We follow the implementation of \cite{LiuCORL2019}.
\subsection{Architecture Details}\label{sec:arch-details}

\noindent\textbf{\twodimlh model.} As discussed in Sec.~\ref{sec:appendix-obs-space}, we have designed the observation given to our agent so that a simple linear layer followed by a soft-max function is sufficient to capture any $f$-restricted policy. As such, our main and auxiliary actor models for this task are simply linear functions mapping the input 6400-dimensional observation to a 4-dimensional output vector followed by a soft-max non-linearity. The critic is computed similarly but with a 1-dimensional output and no non-linearity.

\noindent\textbf{\pd model.} Our \pd model has three sequential components. The first embedding layer maps a given observation, a value in $\{0,1,2,3\}$, to an 128-dimensional embedding. This 128-dimensional vector is then fed into a 1-layer LSTM (with a 128-dimensional hidden state) to produce an 128-output representation $h$. We then compute our main actor policy by applying a $128\times 7$ linear layer followed by a soft-max non-linearity. The auxiliary actor is produced similarly but with separate parameters in its linear layer. Finally the critic's value is generated by applying a $128\times 1$ linear layer to $h$.

\noindent\textbf{\textsc{MiniGrid} model.} Here we detail each component of the model architecture illustrated in~\figref{fig:model}. The encoder (`Enc.') converts observation tensors (integer tensor of shape $7\times7\times3$) to a corresponding embedding tensor via three embedding sets (of length $8$) corresponding to type, color, and state of the object. The observation tensor, which represents the `lights-out' condition, has a unique (\ie, different from the ones listed by~\citep{gym_minigrid})  type, color and state. This prevents any type, color or state from having more than one connotation. The output of the encoder is of size $7\times7\times24$. This tensor is flattened and fed into a (single-layered) LSTM with a 128-dimensional hidden space. The output of the LSTM is fed to the main actor, auxiliary actor, and the critic. All of these are single linear layers with output size of $|\mathcal{A}|$, $|\mathcal{A}|$ and $1$, respectively (main and auxiliary actors are followed by soft-max non-linearities). 

\noindent\textbf{\textsc{PointNav}, \textsc{ObjectNav}, and \textsc{CoopNav} model.}

For the \pnav \cite{habitat19iccv}, \onav \cite{DeitkeEtAl2020}, and \cnav \cite{LiuCORL2019} tasks, we (for fair comparison) use model architectures from prior work. For use with ADVISOR, these model architectures require an additional auxiliary policy head. We define this auxiliary policy head as a linear layer applied to the model's final hidden representation followed by a softmax non-linearity.

\subsection{Fair Hyperparameter Tuning}\label{sec:hyperparameter-tuning}

As discussed in the main paper, we attempt to ensuring that comparisons to baselines are fair. In particular, we hope to avoid introducing misleading bias in our results by extensively tuning the hyperparameters (hps) of our ADVISOR methodology while leaving other methods relatively un-tuned. 

\noindent\textbf{\twodimlh: Tune by Tuning a Competing Method.} The goal of our experiments with the \twodimlh environment are, principally, to highlight that increasing the imitation gap can have a substantial detrimental impact on the quality of policies learned by training IL. Because of this, we wish to give IL the greatest opportunity to succeed and thus we are not, as in our \pd/\minigrid experiments, attempting to understand its expected performance when we must search for good hyperparameters. To this end, we perform the following procedure for every $i,j\in\{1,3,5\ldots, 15\}$ with $i<j$.

For every learning rate $\lambda\in \{\text{100 values evenly spaced in $[10^{-4}, 1]$ on a log-scale}\}$ we train a $f^i$-restricted policy to imitate a $f^j$-optimal teacher using BC. For each such trained policy, we roll out trajectories from the policy across 200 randomly sampled episodes (in the \twodimlh there is no distinction between training, validation, and test episodes as there are only four unique initial world settings). For each rollout, we compute the average cross entropy between the learned policy and the expert's policy at every step. A ``best'' learning rate $\lambda^{i,j}$ is then chosen by selecting the learning rate resulting in the smallest cross entropy (after having smoothed the results with a locally-linear regression model \citep{AllOfNonparametricStatistics}). 

A final learning rate is then chosen as the average of the $\lambda^{i,j}$ and this learning rate is then used when training all methods to produce the plots in \figref{fig:lh-plots}. As some baselines require additional hyperparameter choices, these other hyperparameters were chosen heuristically (post-hoc experiments suggest that results for the other methods are fairly robust to these other hyperparameters). 

\noindent\textbf{\pd and \minigrid tasks: Random Hyperparameter Evaluations.} As described in the main paper, we follow the best practices suggested by~\citet{DodgeGCSS19}. In particular, for our \pd and \minigrid tasks we train each of our baselines when sampling that method's hyperparameters, see Table~\ref{tab:baseline_details} and recall Sec.~\ref{sec:additional-baseline-details}, at random 50 times. Our plots, \eg,  \figref{fig:minigrid-tasks}, then report an estimate of the expected (validation set) performance of each of our methods when choosing the best performing model from a fixed number of random hyperparameter evaluations. Unlike \citep{DodgeGCSS19}, we compute this estimate using a U-statistic \cite[Chapter~12]{vaart} which is unbiased. Shaded regions encapsulate the 25-to-75th quantiles of the bootstrap distribution of this statistic.

\noindent\textbf{\pnav, \onav, and \cnav tasks: use hyperparameters from in prior work.} Due to computational constraints, our strategy for choosing  hyperparameters for the \pnav, \onav, and \cnav tasks was simply to follow prior work whenever possible. Of course, there was no prior work suggesting good hyperparameter values for the $\alpha,\beta$ parameters in our new ADVISOR loss. Following the intuitions we gained from our the \twodimlh, \pd, and \minigrid experiments, we fixed $\alpha,\beta$ to (10, 0.1) for \pnav, $\alpha,\beta$ to (20, 0.1) for \onav, and $\alpha,\beta$ to (0.01, 0) for \cnav. For the \onav task, we experimented with setting $\beta=0$ and found that the change had essentially no impact on performance (validation-set SPL after $\approx100$Mn training steps actually improved slightly from .1482 to .1499 when setting $\beta=0$).

\subsection{Training Implementation}\label{sec:training-imp-details}

\begin{figure}[t]
    \centering
    {
      \phantomsubcaption\label{fig:lh-plots-ppo}
      \phantomsubcaption\label{fig:lh-plots-bc}
      \phantomsubcaption\label{fig:lh-plots-bcppo}
      \phantomsubcaption\label{fig:lh-plots-advisor}
    }
    \includegraphics[width=1\linewidth]{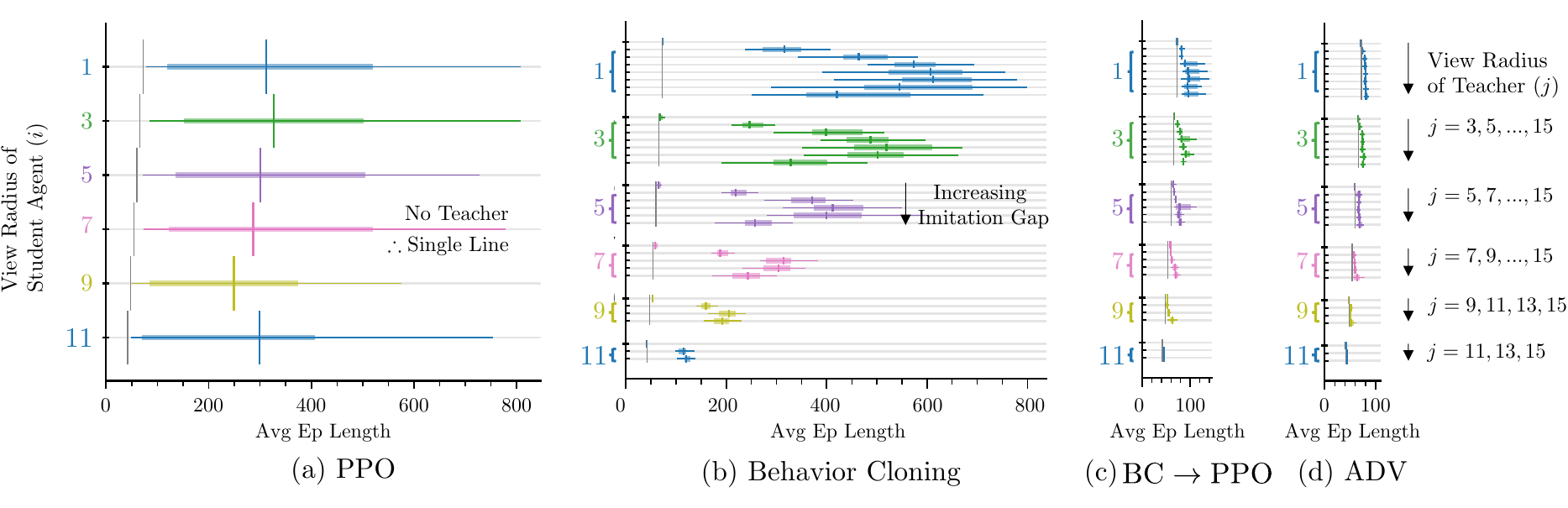}
    \caption{\textbf{Size of the imitation gap directly impacts performance.} Training $f^i$-restricted students with $f^j$-optimal teachers (in \twodimlh).}
    \label{fig:app-lh-plots}
\end{figure}

As discussed previously, for our \pnav, \onav, and \cnav experiments, we have used standard training implementation details (\eg reward structure) from prior work. Thus, in the below, we provide additional details only for the \twodimlh, \pd, and \minigrid tasks.

A summary of the training hyperparameters and their values is included in~\tabref{tab:hyperparams-grid}. Kindly see~\citep{schulman2017proximal} for details on PPO and~\citep{schulman2015high} for details on generalized advantage estimation (GAE). 

\noindent\textbf{Max.\ steps per episode.} The maximum number of steps allowed in the \twodimlh task is 1000. Within the \pd task, an agent can never take more than 11 steps in a single episode (1 action to select the door and then, at most, 10 more actions to input the combination if $d_1$ was selected) and thus we do not need to set a maximum number of allowed steps. The maximum steps allowed for an episode of \textsc{WC/LC} is set by~\citep{gym_minigrid,chevalier2018babyai} to $4 S^2$, where $S$ is the grid size. We share the same limits for the challenging variants -- \textsc{Switch} and \textsc{Corrupt}. Details of task variants, their grid size, and number of obstacles are included in~\secref{sec:additional-task-details}.

\noindent\textbf{Reward structure.} Within the \twodimlh task, the agent receives one of three possible rewards after every step: when the agent finds the goal it receives a reward of 0.99, if it otherwise has reached the maximum number of steps (1000) it receives a $-1$ reward, and otherwise, if neither of the prior cases hold, it obtains a reward of $-0.01$. See Sec.~\ref{sec:poisoned-doors-task-description-full} for a description of rewards for the \pd task. For \textsc{WC/LC},~\citep{gym_minigrid,chevalier2018babyai} configure the environment to give a $0$ reward unless the goal is reached. If the goal is reached, the reward is $1-\frac{\text{episode length}}{\text{maximum steps}}$. We adopt the same reward structure for our \textsc{Switch} and \textsc{Corrupt} variants as well.

\noindent\textbf{Computing infrastructure.} As mentioned in~\secref{sec:evaluation}, for all tasks (except \textsc{LH}) we train $50$ models (with randomly sampled hps) for each baseline. This amounts to $750$ models per task or $6700$ models in total. For each task, we utilize a \texttt{g4dn.12xlarge} instance on AWS consisting of $4$ NVIDIA T4 GPUs and $48$ CPUs. We run through a queue of $750$ models using $\approx 40$ processes. For tasks set in the \minigrid environments, models each require $\approx 1.2$ GB GPU memory and all training completes in $18$ to $36$ hours. For the \pd task, model memory footprints are smaller and training all models is significantly faster ($<8$ hours).

\begin{figure}[t]
    \centering
    {
        \phantomsubcaption\label{fig:lc-fs-supp}
        \phantomsubcaption\label{fig:wc-supp}
        \phantomsubcaption\label{fig:wc-os-supp}
        \phantomsubcaption\label{fig:wc-fs-supp}
        \phantomsubcaption\label{fig:wc-ce-supp}
        \phantomsubcaption\label{fig:lc-fs-robust-supp}
        \phantomsubcaption\label{fig:wc-robust-supp}
        \phantomsubcaption\label{fig:wc-os-robust-supp}
        \phantomsubcaption\label{fig:wc-fs-robust-supp}
        \phantomsubcaption\label{fig:wc-ce-robust-supp}
    }
    \includegraphics[width=\linewidth]{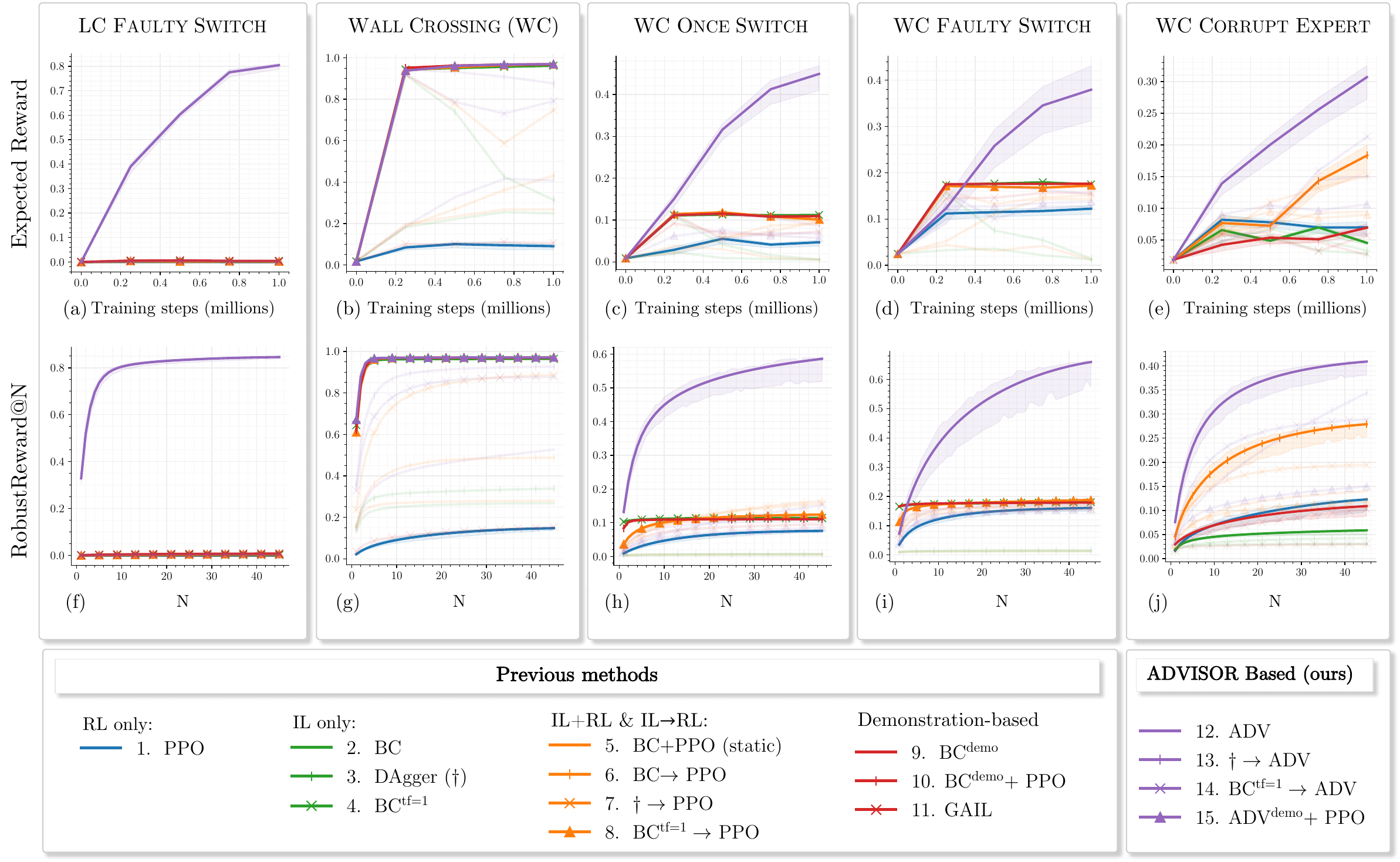}
    \caption{\textbf{Additional results for \minigrid tasks.} Here we include the results on the \minigrid tasks missing from \Cref{fig:minigrid-tasks}.}
    \label{fig:minigrid-tasks-supp}
\end{figure}

\begin{figure*}[t]
    \centering
    \begin{tabular}{cc}
        \includegraphics[trim=5 5 10 10,clip,width=0.3\textwidth]{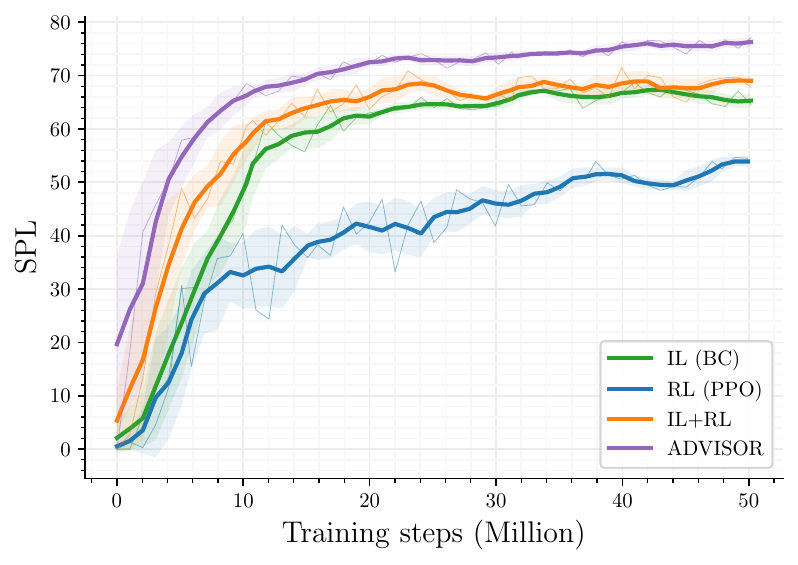} &
        \includegraphics[trim=5 5 10 10,clip,width=0.3\textwidth]{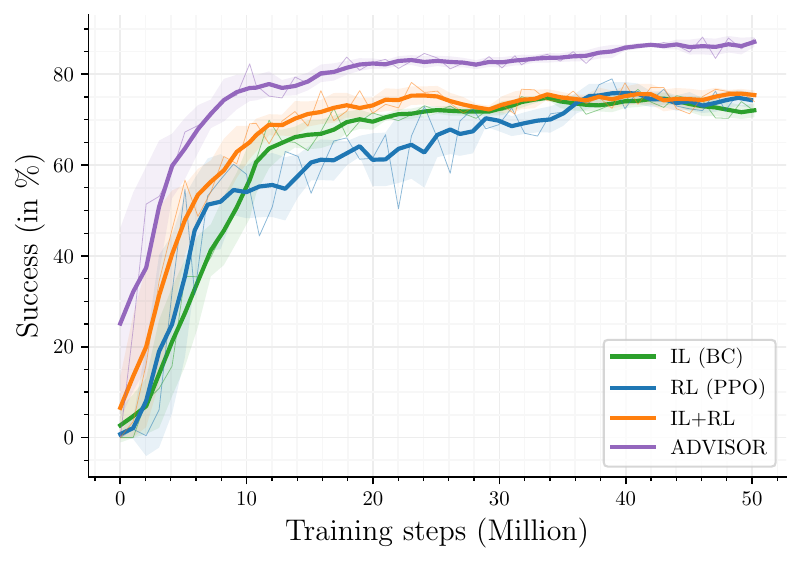}\\
    \end{tabular}
    \caption{\textbf{Learning curves for \pnav.} SPL (scaled by 100) and success rate (in \%) are plotted \vs training steps, following the standard protocols. We evaluate checkpoints after every $1024$k frames of experience. This is plotted as the thin line. The thick line and shading depicts the rolling mean (with a window size of $2$) and corresponding standard deviation.
    }
    \label{fig:pointgoal-plots}
\end{figure*}

\begin{figure*}[t]
    \centering
    {
    \phantomsubcaption\label{fig:cnav-plots-range0.8}
    \phantomsubcaption\label{fig:cnav-plots-range1.2}
    \phantomsubcaption\label{fig:cnav-plots-range1.6}
    \phantomsubcaption\label{fig:cnav-plots-range2.0}
    }
    \includegraphics[width=\textwidth]{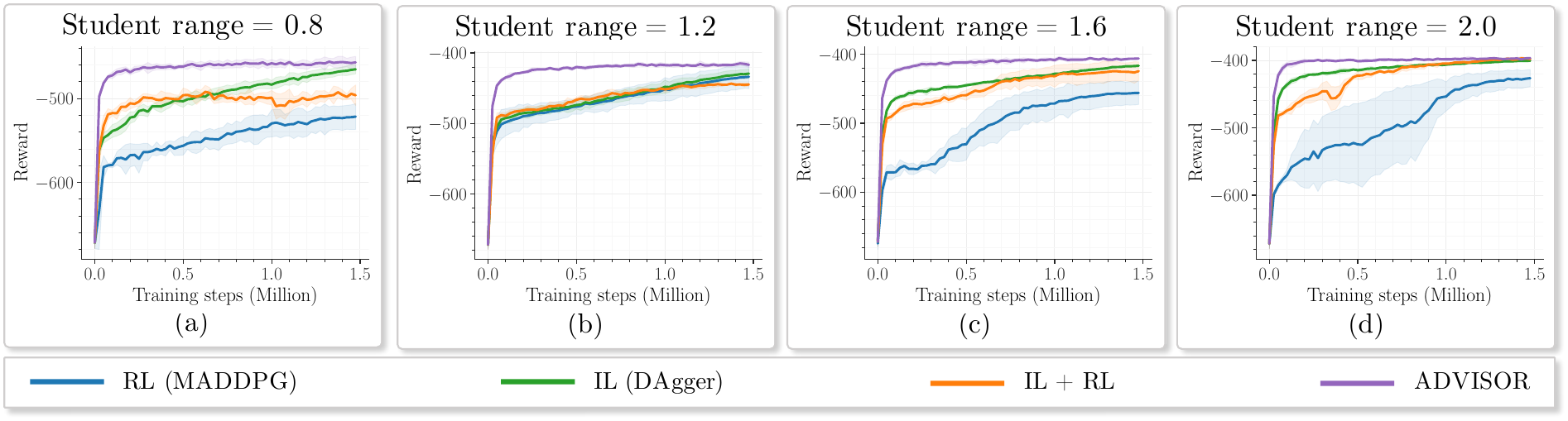}
    \caption{\textbf{Learning curves for \cnav.} Rewards are plotted \vs training steps, following the standard protocols. For a full-range teacher, we train students with different (and limited) visibility range of $0.8, 1.2, 1.8$, and $2.0$. The networks are initialized with four different seeds and the mean and standard deviation are plotted. Checkpoints are evaluated at every 25k steps.
    }
    \label{fig:cnav-plots}
\end{figure*}

\subsection{Additional results}\label{sec:more-plots}

Here we record additional results that were summarized or deferred in \Cref{sec:results}. In particular,

\begin{itemize}
    \item \Cref{fig:app-lh-plots} complements \Cref{fig:lh-plots} from the main paper and provides results for additional baselines on the \twodimlh task. Notice that both the pipelined IL$\to$PPO and ADVISOR methods greatly reduce the impact of the imitation gap (\Cref{fig:lh-plots-bcppo,fig:lh-plots-advisor} versus \Cref{fig:lh-plots-bc}) but our ADVISOR method is considerably more effective in doing so (\Cref{fig:lh-plots-bcppo} v.s. \Cref{fig:lh-plots-advisor}).
    
    \item \Cref{fig:minigrid-tasks-supp} shows the results on our remaining \minigrid tasks missing from \Cref{fig:minigrid-tasks}. Notice that the trends here are very similar to those from \Cref{fig:minigrid-tasks}, ADVISOR-based methods have similar or better performance than our other baselines.
    
    \item \Cref{tab:pd-and-mingrid-results-all} shows an extended version of \Cref{tab:pd-and-mingrid-results} where, rather than grouping methods together, we display results for each method individually.
    
    \item \Cref{fig:pointgoal-plots} displays validation set performance of our \pnav baselines over training. Note that static combination of RL and IL losses improves individual RL/IL baselines. Our adaptive combination of these losses (ADVISOR) outperforms these baselines and is more sample efficient.
    
    \item \Cref{fig:cnav-plots} lays out the performance of agents on the \cnav task. In the main paper we include results for the limited visibility range of 1.6 for the student. Here, we include results for four visibility range. RL only baseline is least sample-efficient. Overall, we find ADVISOR is significantly more sample efficient -- most of the learning is completed in just $0.2$ million steps while the other baselines take over $1.5$ million steps.
\end{itemize}

\begin{table}[t]
    \centering
    \scriptsize
    \setlength{\tabcolsep}{10pt}
    \begin{tabular}{ll}
        \toprule
        \textbf{Hyperparamter} &  \textbf{Value}\\\hline
        \multicolumn{2}{c}{\textit{Structural}}\\\hline
        Cell type embedding length & $8$\\
        Cell color embedding length & $8$\\
        Cell state embedding length & $8$\\
        RNN type & LSTM\\
        RNN layers & $1$\\
        RNN hidden size & $128$\\
        \# Layers in critic & $1$\\
        \# Layers in actor & $1$\\
        \hline\multicolumn{2}{c}{\textit{PPO}}\\\hline
        Clip parameter ($\epsilon$)~\citep{schulman2017proximal} & $0.1$\\
        Decay on $\epsilon$ & $\mathtt{Linear}(1,0)$\\
        Value loss coefficient & $0.5$\\
        Discount factor ($\gamma$) & $0.99$\\
        GAE parameter ($\lambda$) & $1.0$\\
        \hline\multicolumn{2}{c}{\textit{Training}}\\\hline
        Rollout timesteps & $100$\\
        Rollouts per batch & 10 \\
        \# processes sampling rollouts & 20  \\
        Epochs & $4$\\
        Optimizer & Adam~\citep{KingmaICLR2015adam}\\
        $(\beta_{1}, \beta_{2})$ for Adam & $(0.9, 0.999)$\\
        Learning rate & \texttt{searched}\\
        Gradient clip norm & $0.5$ \\
        Training steps (\textsc{WC/LC} \& variants) & $1\cdot10^{6}$\\
        Training steps (\twodimlh \& \pd) & $3\cdot10^{5}$\\
        \bottomrule
    \end{tabular}
    \caption{Structural and training hyperparameters for \twodimlh, \pd, and \minigrid tasks.} \label{tab:hyperparams-grid}
\end{table}

\begin{table}[t]
    \centering
    \scriptsize
    \setlength{\tabcolsep}{10pt}
    \begin{tabular}{lcc}
        \toprule
        \textbf{Hyperparamter} &  \pnav & \onav\\\hline
        \multicolumn{3}{c}{\textit{Structural}}\\\hline
        RNN type & \multicolumn{2}{c}{GRU}\\
        RNN layers & \multicolumn{2}{c}{$1$}\\
        RNN hidden size & \multicolumn{2}{c}{512}\\
        \# Layers in critic & \multicolumn{2}{c}{$1$}\\
        \# Layers in actor & \multicolumn{2}{c}{$1$}\\
        \hline\multicolumn{3}{c}{\textit{PPO}}\\\hline
        Clip parameter ($\epsilon$)~\citep{schulman2017proximal} & \multicolumn{2}{c}{$0.1$}\\
        Decay on $\epsilon$ & \multicolumn{2}{c}{None}\\
        Value loss coefficient & \multicolumn{2}{c}{$0.5$}\\
        Discount factor ($\gamma$) & \multicolumn{2}{c}{$0.99$}\\
        GAE parameter ($\lambda$) & \multicolumn{2}{c}{$0.95$}\\
        \hline\multicolumn{3}{c}{\textit{Training}}\\\hline
        Rollout timesteps & \multicolumn{2}{c}{128}\\
        Rollouts per batch & 60 & 8 \\
        \# processes sampling rollouts & 60 & 16 \\
        Epochs & \multicolumn{2}{c}{$4$}\\
        Optimizer & \multicolumn{2}{c}{Adam~\citep{KingmaICLR2015adam}}\\
        $(\beta_{1}, \beta_{2})$ for Adam & \multicolumn{2}{c}{$(0.9, 0.999)$}\\
        Learning rate & $3\cdot 10^{-4}$ & $2.5 \cdot 10^{-4}$\\
        Gradient clip norm & $0.5$ & $0.1$ \\
        Training steps & $100\cdot10^{6}$ & $50\cdot10^{6}$\\
        \bottomrule
    \end{tabular}
    \caption{Structural and training hyperparameters for \pnav and \onav.} \label{tab:hyperparams-onav}
\end{table}

\begin{table}[t]
\centering
\resizebox{\textwidth}{!}{%
\begin{tabular}{l>{\columncolor{ColorPNav}}c>{\columncolor{ColorFurnMove}}c>{\columncolor{ColorFurnMove}}c>{\columncolor{ColorFurnMove}}c>{\columncolor{ColorFurnMove}}c>{\columncolor{ColorFootball}}c>{\columncolor{ColorFootball}}c>{\columncolor{ColorFootball}}c>{\columncolor{ColorFootball}}c>{\columncolor{ColorFootball}}c>{\columncolor{ColorFootball}}cc}
\toprule
\multicolumn{1}{c}{Tasks $\rightarrow$} & \textbf{\textsc{PD}} & \multicolumn{4}{c}{\cellcolor{ColorFurnMove}\textbf{\textsc{LavaCrossing}}} & \multicolumn{4}{c}{\cellcolor{ColorFootball}\textbf{\textsc{WallCrossing}}}  \\
 Training routines $\downarrow$  & - & Base Ver. & Corrupt Exp. & Faulty Switch & Once Switch & Base Ver. & Corrupt Exp. & Faulty Switch & Once Switch  \\ \midrule
PPO & 0 & 0 & 0 & 0.01 & 0 & 0.09 & 0.07 & 0.12 & 0.05 \\
BC & -0.6 & 0.1 & 0.02 & 0 & 0 & 0.25 & 0.05 & 0.01 & 0.01 \\
DAgger $(\dagger)$ & -0.59 & 0.14 & 0.02 & 0 & 0 & 0.31 & 0.03 & 0.01 & 0.01 \\
BC$^{\text{tf}=1}$ & -0.62 & 0.88 & 0.02 & 0.02 & 0 & 0.96 & 0.03 & 0.17 & 0.11 \\
BC$+$PPO (static) & -0.59 & 0.12 & 0.08 & 0 & 0 & 0.27 & 0.09 & 0.01 & 0 \\
BC$ \to$ PPO & -0.17 & 0.15 & 0.32 & 0.02 & 0 & 0.43 & 0.18 & 0.14 & 0.09 \\
$\dagger \to$ PPO & -0.45 & 0.32 & 0.61 & 0.02 & 0 & 0.75 & 0.15 & 0.15 & 0.1 \\
BC$^{\text{tf}=1} \to$ PPO & -0.5 & 0.94 & 0.74 & 0.04 & 0 & \textbf{0.97} & 0.09 & 0.17 & 0.1 \\
BC$^{\text{demo}}$ & -0.62 & 0.88 & 0.02 & 0.02 & 0 & 0.96 & 0.07 & 0.18 & 0.11 \\
BC$^{\text{demo}} +$ PPO & -0.64 & \textbf{0.96} & 0.2 & 0.02 & 0 & \textbf{0.97} & 0.03 & 0.17 & 0.11 \\
GAIL & -0.09 & 0 & 0 & 0.02 & 0 & 0.11 & 0.06 & 0.16 & 0.07 \\
ADV & \textbf{1} & 0.18 & 0.8 & \textbf{0.77} & \textbf{0.8} & 0.41 & \textbf{0.31} & \textbf{0.38} & \textbf{0.45} \\
BC$^{\text{tf}=1} \to$ ADV & -0.13 & 0.55 & 0.83 & 0.02 & 0 & 0.88 & 0.15 & 0.15 & 0.09 \\
$\dagger \to$ ADV & -0.1 & 0.47 & 0.73 & 0.01 & 0 & 0.79 & 0.21 & 0.13 & 0.07 \\
ADV$^{\text{demo}} +$ PPO & 0 & \textbf{0.96} & \textbf{0.94} & 0.03 & 0 & \textbf{0.97} & 0.11 & 0.14 & 0.06 \\ \bottomrule
\end{tabular}
}%
\caption{\textbf{Expected rewards for the \textsc{PoisonedDoors} task and \textsc{MiniGrid} tasks.} Here we show an expanded version of Table \ref{tab:pd-and-mingrid-results-all} where results for all methods rather than grouped methods. For each of our 15 training routines we report the expected maximum validation set performance (when given a budget of 10 random hyperparameter evaluations) after training for $\approx$300k steps in the \textsc{PoisonedDoors} environment and $\approx$1Mn steps in our 8 \minigrid tasks. The maximum possible reward is 1 for the \minigrid tasks.} \label{tab:pd-and-mingrid-results-all}
\end{table}

\end{document}